\PassOptionsToPackage{dvipsnames}{xcolor}
\documentclass{article}
\usepackage{kr}

\usepackage{times}
\usepackage{soul}
\usepackage{url}
\usepackage[hidelinks]{hyperref}
\usepackage[utf8]{inputenc}
\usepackage[small]{caption}
\usepackage{graphicx}
\usepackage{amsmath}
\usepackage{amsthm}
\usepackage{booktabs}
\usepackage{algorithm}
\usepackage{algorithmic}
\newif\ifappendix
\appendixtrue

\newenvironment{example*}[2][]
	{\smallskip \pagebreak[2] \par \noindent \textbf{Example~\ref{#2} (continued).}\it}{\par\smallskip}

\usepackage{multirow}
\usepackage{multicol}

\usepackage{amssymb}
\usepackage{complexity}
\usepackage{thm-restate}
\usepackage{todonotes}
\usepackage{tikz,pgf}
\usetikzlibrary{shapes,backgrounds,shapes.geometric,shadows,patterns}
\usetikzlibrary{positioning}
\tikzstyle{arg}=[draw,circle,inner sep=1pt,minimum size=.5cm]

\usepackage{etoolbox}
\usepackage{thm-restate}
\usepackage{todonotes}

\newtheorem{theorem}{Theorem}[section]

\newtheorem{definition}[theorem]{Definition}
\newtheorem{remark}[theorem]{Remark}

\newtheorem{example}[theorem]{Example}

\newcommand{\contrary}[1]{\overline{#1}}
\newcommand{\contraryempty}{\contrary{\phantom{a}}}
\newcommand{\asms}{asms}

\newcommand{\ABAdef}{D=(\mathcal{L},\mathcal{R},\mathcal{A},\contraryempty)}
\newcommand{\lit}{\ensuremath{\mathcal{L}}}
\newcommand{\rules}{\ensuremath{\mathcal{R}}}
\newcommand{\asm}{\ensuremath{\mathcal{A}}}
\newcommand{\contraryc}[1]{\overline{#1}}

\newcommand{\BF}{\mathit{F}}
\newcommand{\la}{\leftarrow}

\usepackage{cancel}

\newcommand{\geqels}{\succeq_{els}}
\newcommand{\eqels}{\simeq_{els}}
\newcommand{\geqsup}{\succeq_{sup}}
\newcommand{\geqset}{\succeq_{set}}
\newcommand{\eqsup}{\simeq_{sup}}
\newcommand{\eqset}{\simeq_{set}}
\newcommand{\setag}{\zeta}

\newcommand{\Att}{\mathit{Att}}
\newcommand{\Sup}{\mathit{Sup}}

\newcommand{\ocr}{\textbf{(OR)}}
\newcommand{\fact}{\textbf{(F)}}
\newcommand{\neut}{\textbf{(N)}}
\newcommand{\ide}{\textbf{(ID)}}
\newcommand{\wl}{\textbf{(WL)}}

\newcommand{\void}{\textbf{(V)}}
\newcommand{\si}{\textbf{(SI)}}
\newcommand{\dom}{\textbf{(D)}}

\newcommand{\nonmax}[1]{#1^\circ} 
\newcommand{\nonmaxm}{\mathit{sup}} 
\newcommand{\spres}{\textbf{($\succeq$-P)}}
\newcommand{\basescore}{\tau}

\title{On Gradual Semantics for Assumption-Based Argumentation}

\author{
Anna Rapberger\and 
Fabrizio Russo\and
Antonio Rago\And 
Francesca Toni\\
\affiliations
Imperial College London, Department of Computing\\
 \emails
\{a.rapberger, fabrizio, a.rago, ft\}@imperial.ac.uk
}

\begin{document}

\maketitle

\begin{abstract}
In computational argumentation, gradual semantics are fine-grained alternatives to extension-based and labelling-based semantics
. They ascribe a dialectical strength to (components of) arguments sanctioning their degree of acceptability. Several gradual semantics have been studied for abstract, bipolar and quantitative bipolar argumentation frameworks (QBAFs), 
as well as, to a lesser extent, for some forms of structured argumentation. 
However, this has not been the case for assumption-based argumentation (ABA), despite it being a popular form of structured argumentation with several applications where gradual semantics could be useful. In this paper, we fill this gap and propose a family of novel gradual semantics for equipping assumptions, which are the core components in ABA frameworks, with dialectical strengths. To do so, we use bipolar set-based argumentation frameworks as an abstraction of (potentially non-flat) ABA frameworks and generalise state-of-the-art modular gradual semantics for QBAFs. We show that our \emph{gradual ABA semantics} satisfy suitable adaptations of desirable  properties of gradual QBAF semantics, such as \emph{balance} and \emph{monotonicity}. We also explore an argument-based
approach that leverages established QBAF modular semantics directly, and use it as baseline. Finally, we conduct experiments with synthetic ABA frameworks to compare our gradual ABA semantics with its argument-based counterpart and assess convergence. 
\end{abstract}

\section{Introduction}

Computational argumentation is a central research area within the field of knowledge representation and reasoning.
Among the most prominent structured argumentation formalisms is \emph{assumption-based argumentation (ABA)}~\cite{DBLP:journals/ai/BondarenkoDKT97,CyrasFST2018}, which provides a versatile framework for modelling and evaluating argumentative reasoning, with applications in several domains, such as recommendations within healthcare~\cite{DBLP:journals/argcom/CyrasOKT21,DBLP:conf/ratio/SkibaTW24,DBLP:conf/atal/ZengSCLWCM20} and explainable causal discovery~\cite{russo2024argumentativecausaldiscovery}. 
ABA frameworks (ABAFs) have been extensively studied, including investigations of their formal properties~\cite{DBLP:journals/jair/CaminadaS17,DBLP:conf/jelia/BertholdRU23}, computational complexity~\cite{DBLP:journals/ai/CyrasHT21,DBLP:journals/jair/LehtonenWJ21,DBLP:conf/kr/LehtonenR0W23} and implementations~\cite{Toni13,DBLP:journals/tplp/LehtonenWJ21,DBLP:conf/ijcai/LehtonenRT0W24}.
The building blocks in ABA are \emph{assumptions}, which are defeasible elements that can be challenged via a so-called \emph{contrary function}, and \emph{inference rules}. 
Relations between assumptions can be conflicting or supporting, where, informally: a set of assumptions $S$ \emph{attacks} an assumption $a$ if the contrary of $a$ is derivable (via a tree-derivation, see Section~\ref{subsec:aba bg}) from (a subset of) $S$; and
$S$ \emph{supports} $a$ if $a$ is derivable from (a subset of) $S$.\footnote{Supports emerge only if the ABAF is not \emph{flat}, i.e., 
 if assumptions 
 are derivable~\cite{DBLP:journals/ai/BondarenkoDKT97}.}
To date, existing research in the evaluation of ABA frameworks focuses exclusively on two main strands. The first is \emph{extension-based semantics}~\cite{DBLP:journals/ai/BondarenkoDKT97}, which use attack and support relations to decide which sets of assumptions can be jointly accepted. The second is \emph{labelling-based semantics}~\cite{DBLP:journals/ijar/SchulzT17}, which categorise each assumption as accepted, rejected, or undefined.
Both forms of semantics may be used to characterise multiple possible 
solutions for tasks such as decision making. 

Gradual semantics is an orthogonal evaluation paradigm with growing popularity, most prominently in the realm of (quantitative, bipolar) abstract argumentation~\cite{CayrolL05,BARONI2019252,DBLP:conf/kr/AmgoudD18a}.
Abstract argumentation represents argumentative reasoning processes via directed graphs in which nodes correspond to abstract arguments and different edge types indicate, for example, attack~\cite{Dung95} and support~\cite{DBLP:journals/ijis/AmgoudCLL08} relations between arguments.
In contrast to extension-based semantics, gradual semantics assign dialectical strengths to arguments, which can be seen as acceptability degrees
, providing a finer-grained evaluation of arguments 
than the one given by extensions~\cite{Dung95,CayrolL05} or labellings \cite{Caminada_09,Gonzalez_21}.
In quantitative bipolar argumentation frameworks (QBAFs)~\cite{BARONI2019252}, gradual semantics compute a numerical \textit{acceptability degree} 
for each argument based on its initial strength (the base score, which reflects the argument's intrinsic plausibility) and 
the acceptability degrees of its attackers and supporters
. Gradual semantics have numerous applications in domains such as engineering decision support \cite{DBLP:conf/kr/RagoTAB16}, social media analysis \cite{Leite_11,Oluokun_24},
online review aggregation \cite{Rago_25} and explainability \cite{Ayoobi_23}, amongst others.
Modular gradual semantics, originally proposed by~\citeauthor{DBLP:journals/corr/abs-1807-06685}~\shortcite{DBLP:journals/corr/abs-1807-06685}, 
decompose these semantics into aggregation and influence functions, 
enabling a systematic study of their formal behaviour (see also~\cite{BaroniCG18,DBLP:conf/atal/Potyka19,PotykaB24}). 

Despite the rich literature on both ABA and 
gradual semantics, there is currently no study investigating gradual semantics 
for ABA in contexts where assumptions themselves carry weights.
This leaves a significant gap in the ABA theory, and limits its application exclusively to scenarios where multiple solutions are tolerable and a numerical evaluation of assumptions is not needed. However, many real-world scenarios require differentiating between assumptions based on weights that capture, for example, source credibility or  empirical evidence (as, e.g., in the 
use of ABA for causal discovery~\cite{russo2024argumentativecausaldiscovery}). 
\begin{example}\label{exm:intro}
 Consider a simple ABAF 
 with inference rule $q \leftarrow b,c$,
 where $a,b,c$ are assumptions, the contrary of $a$ is $q$ 
 and the contraries of $b$ and $c$ are sentences $x,y$ respectively.
 This ABAF may represent a very simple decision-making setting:
 The 
 assumption $a$ stands for `having a picnic',   
 $q$ stands for `having a picnic is a bad idea', $b$ stands for `bad weather is forecasted' and $c$ stands for `no friends are around'.
 Given that $x,y$ are not derivable and thus $b,c$ cannot be attacked,
 any extension- or labelling-based semantics will return the set $\{b,c\}$, indicating that $a$ is not accepted. 
 Suppose that, while only being able to treat $b$ and $c$ as assumptions (as they are about the future),  we have some quantitative information (weights) about them, e.g., from a weather forecasting website and from the profile of the friends from a social network. 
 If the weights of $b$ and $c$ are both high (because the website and profiles are both reliable),
 then we would expect a low degree of acceptability for $a$, but if they are both low, the degree should  be higher.
\end{example}
With the standard extension-based~\cite{DBLP:journals/ai/BondarenkoDKT97}
or labelling-based~\cite{DBLP:journals/ijar/SchulzT17} ABA semantics, the numerical information available in such settings cannot be processed.
In this paper, we take a foundational step towards closing this gap by exploring gradual semantics for \emph{weighted assumptions} in ABA. 
To this end, we utilise 
\emph{bipolar set-based argumentation frameworks} (BSAFs)~\cite{DBLP:conf/kr/BertholdR024}, 
an abstract formalism that provides a concise representation for ABAFs by modelling the collective attacks and collective supports.

\begin{example*}{exm:intro}
The BSAF representation contains a node for each assumption and a single attack 
\textcolor{magenta}{$(\{b,c\},a)$},
depicted below (left). The table (right) contains possible values (base scores) for the assumptions.
\begin{center}
\begin{minipage}[c]{0.2\textwidth}
\begin{center}
		\begin{tikzpicture}[scale=1,>=stealth]
        \node at (-1.2,0.1) {BSAF:};
		\path
		(0,0)node[arg] (a){$a$}
        (-.8,-.7)node[arg] (b){$b$}
		(.8,-.7)node[arg] (c){$c$}
        ;
	
\path[thick,->,magenta]
(b)edge[out=0,in=-90](a)
(c)edge[out=180,in=-90](a)
;
            
\end{tikzpicture}
\end{center}
\end{minipage}
\hspace{10pt}
\begin{minipage}[c]{0.2\textwidth}
\begin{tabular}{c | c c c}
    &  a & b & c \\
     \hline \vspace{-8pt} \\
   $\tau_1$ & 1 & 1 & 1 \\
   $\tau_2$ & 1 & 0.1 & 0.2 \\
\end{tabular}
\end{minipage}    
\end{center}
Intuitively, the evaluation with respect to $\tau_1$ should result in a much lower dialectical strength of assumption $a$ than $\tau_2$.
\end{example*}

We employ BSAFs for their compactness and their flexibility to focus only on the defeasible elements of ABAFs, assumptions, whose granular evaluation is the focus of our proposed semantics.
Our main contributions are as follows:
\begin{itemize}
    \item We introduce a family of gradual semantics for ABA and investigate whether they satisfy fundamental properties of gradual semantics, such as balance and monotonicity, as well as their convergence behaviour. We introduce bespoke assumption aggregation functions and study their behaviour in relation to the other components of the modular semantics.
    \item
    Orthogonal to our (assumption-based) gradual ABA semantics, 
    we craft an (argument-based) QBAF baseline method to evaluate an ABAF with weighted assumptions, building on established approaches. 
    For this we leverage the correspondence between flat ABAFs and AFs to then employ existing gradual QBAF semantics. 

    \item We conduct an experimental evaluation to study the convergence behaviour of our gradual ABA semantics and compare it to the QBAF baseline. Among our findings is that instantiations of the gradual ABA semantics converge in 90\% of the scenarios while the QBAF baseline does so only in about 70\% of the scenarios, demonstrating the advantages of our proposed native ABA approach on the existing AF-based pipeline. Code is provided at \url{https://github.com/briziorusso/GradualABA}.
\end{itemize}
All omitted proofs and experimental details can be found in the appendix\ifappendix. \else\ of the technical report~\cite{arxivABAgrad}.
\fi

\newcommand{\allargs}{\mathit{Args}}
\newcommand{\aargs}{\mathbf{A}}
\newcommand{\aatts}{\mathbf{R}}
\newcommand{\asupps}{\mathbf{S}}
\newcommand{\multiS}{\mathbb{S}}
\newcommand{\multiA}{\mathbb{A}}

\newcommand{\AFargs}{\aargs}
\newcommand{\AFatts}{\aatts}
\newcommand{\maxval}{m}

\newcommand{\Pos}{\mathit{Pos}}
\newcommand{\neutr}{\textbf{(N)}}
\newcommand{\mon}{\textbf{(M)}}
\newcommand{\bal}{\textbf{(B)}}
\newcommand{\AFroute}{\textbf{(II)}}
\newcommand{\nativeroute}{\textbf{(I)}}

\section{Related Work}
The use of gradual semantics within structured argumentation frameworks in general has only recently started to receive attention.
Within the context of ABA, one recent work \cite{Skiba_23} explored the use of a family of ranking-based semantics for assumptions, reducing flat ABA frameworks to abstract argumentation and then applying existing ranking-based semantics, an endeavour which is orthogonal to our definition of purpose-built gradual semantics for general (possibly non-flat) ABA.  
To the best of our knowledge, there are no other works which explore gradual semantics within ABA.

Gradual semantics have also appeared in other forms of structured argumentation.
Heyninck et al.~\shortcite{Heyninck_23} explore the use of gradual semantics for \emph{deductive argumentation} \cite{besnard2014constructing}. 
Here, the focus is on relationships between properties of gradual semantics and culpability (of any logical inconsistency) in abstract argumentation frameworks extracted from deductive argumentation. 
Further, in a restricted type of structured argumentation, namely \emph{statement graphs} \cite{hecham2018flexible}, Jedwabny et al.~\shortcite{Jedwabny_20} define a modular gradual semantics based on complete support (inference) trees, demonstrating that one instance satisfies several useful properties centred around cases with complete information.
Another study~\cite{Rago_25X} explores the deployment of gradual semantics from abstract argumentation in these statement graphs, demonstrating that they present distinct advantages in cases under incomplete information.
Finally, Amgoud and Ben-Naim~\shortcite{Amgoud_15} do not explicitly introduce a quantitative strength, but they investigate ranking-based semantics for logic-based argumentation frameworks.

Several works relate to ours in terms of aggregating the strengths of arguments.
In the realm of ASPIC$^+$~\cite{Modgil_14}, 
Prakken~\shortcite{PRAKKEN2024104193} provides a comprehensive study of argument strength, instantiated with a formal model, considering the different ways in which an argument can be attacked.
Spaans~\shortcite{Spaans2021} explores initial argument strength aggregation functions in ASPIC$^+$, considering several principles similar in spirit to those in Section~\ref{subsec:aggregate sets BSAF}.
Other works use probabilistic argument strength in the context of logic-based argumentation, as studied by e.g., Prakken~\shortcite{prakken2018probabilistic} in ASPIC$^+$ and 
Hunter~(\citeyear{hunter2022argument,HUNTER201347}).
\citeauthor{DBLP:journals/argcom/RossitMDM21}~(\citeyear{DBLP:journals/argcom/RossitMDM21}) study accrual in quantitative abstract argumentation. 
In contrast to our work, they study accrual strengths in the context of extension-based semantics.

Lastly, ABA has been generalised to accommodate preferences~\cite{CyrasT16} and priorities~\cite{DungT10,Fan23}. These works incorporate numbers to evaluate acceptance. However, in contrast to gradual semantics, these approaches use preferences and priorities to modify the argumentation frameworks, e.g., by breaking symmetric attacks or by reversing attacks, to subsequently calculate extensions. Instead, gradual semantics propagate strengths across the entire framework with the final scores reflecting not only the initial (base) scores (which one could compare to preferences or priorities), but the whole structure
.

\section{Background}
\newcommand{\cl}{\textit{cl}}
\newcommand{\aset}{E}
\newcommand{\bset}{E'}
\subsection{Assumption-Based Argumentation}\label{subsec:aba bg}
We recall assumption-based argumentation (ABA)~\cite{CyrasFST2018}.
We assume a deductive system $(\lit,\mathcal{R})$, where  $\lit$ is a formal language, i.\,e., a set of sentences, 
and $\mathcal{R}$ is a set of inference rules over $\lit$. A rule $r \in \mathcal{R}$ has the form
$a_0 \leftarrow a_1,\ldots,a_n$ with $a_i \in \lit$.
We write $head(r) = a_0$ and $body(r) = \{a_1,\ldots,a_n\}$ for the (possibly empty) body of $r$.%
\begin{definition}
	An ABA framework (ABAF) is a tuple $(\lit,\asm,\rules,\contraryempty)$ with deductive system $(\mathcal{L},\mathcal{R})$, a non-empty set $\mathcal{A} \subseteq \mathcal{L}$ of assumptions, and  
contrary function $\contraryempty:\mathcal{A}\rightarrow \mathcal{L}$.
\end{definition}

Unless otherwise specified, we assume a unique contrary $a_c$ for each assumption  $a\in \mathcal{A}$.
We write $\contrary{a}$ to denote $a_c$.

We fix an arbitrary ABAF $\ABAdef$. $D$ is \emph{flat} if $head(r)\notin \asm$ for any $r \in \mathcal{R}$.
A sentence $p \in \mathcal{L}$ is \emph{tree-derivable} via tree $t$ from assumptions $\aset \subseteq \mathcal{A}$ and rules $R \subseteq \mathcal{R}$, denoted by $\aset \vdash_{t,R} p$ (tree-derivation), if $t$ is a finite rooted labelled tree such that the root is labelled with $p$, the set of labels for the leaves of $t$ is equal to $\aset$ or $\aset \cup \{\top\}$, and 
there is a surjective mapping from the set of internal nodes to $R$, satisfying that
for each internal node $v$, there exists $r\in R$ s.t.\ $v$ is labelled with $head(r)$, and the set of all successor nodes corresponds to $body(r)$ or $\top$ if $body(r)=\emptyset$.
We say $p$ is derivable (in $D$) iff there is a tree-derivation $\aset \vdash_{R,t} p$, 
$\aset \subseteq \mathcal{A}$, $R\subseteq \mathcal{R}$.
We drop $R$ or $t$ if clear from the context.
\begin{definition}
    The derivation $\aset\vdash p$ is an \emph{ABA argument} if there is a tree-derivation $\aset\vdash_{R,t} p$  for $\aset \subseteq \mathcal{A}$, $R\subseteq \mathcal{R}$. $\allargs_D=\{(A,p)\mid \aset \vdash_{R,t} p\}$ is the set of all arguments. 
\end{definition}
We call arguments of the form $\{a\}\vdash a$, where $a\in\mathcal{A}$, assumption arguments; 
and arguments with $R\neq \emptyset$ rule-based arguments.
We identify assumption arguments with their assumptions.
Given an argument $x=\aset \vdash p$, by $\cl(x)=p$ we denote the claim (conclusion) of an argument, by $\asms(x)=\aset$ its assumptions.
Given $\aset\subseteq \mathcal{A}$, we let $\contrary{\aset}=\{\contrary{a}\mid a\in \aset\}$.
$\aset$ is an attacker of $a\in\mathcal{A}$ iff $\aset\vdash \contrary{a}$;
$\aset$ is a supporter of $a$ iff there is $\aset \vdash_R a$ and $R\neq \emptyset$. 
Given $a\in\mathcal{A}$, we let $\Att(a)=\{ \aset \mid \aset \vdash \contrary{a}\}$
denote the set of all attackers and $\Sup(a)= \{\aset \mid \aset \vdash a\}$ all supporters. 
We extend the functions to sets of assumptions:
for $\aset\subseteq \mathcal{A}$, let $\Att(\aset)=\bigcup_{a\in \aset}\Att(a)$;
analogously for $\Sup(\aset)$.

\newcommand{\AF}{\mathcal{F}}
\subsection{Abstract Representations of ABA}
Abstract argumentation~\cite{Dung95} considers graph-like systems where edges encode relations between abstract entities, typically called arguments (we will also associate them with assumptions).
We recall bipolar set-argumentation frameworks (BSAFs)~\cite{DBLP:conf/kr/BertholdR024} that capture (collective) set-attacks and -supports.
\begin{definition}
    A bipolar set-argumentation framework (BSAF) is a tuple $\BF = (\aargs,\aatts,\asupps)$, 
    where $\aargs$ is a finite set of arguments, $\aatts\subseteq 2^\aargs\times \aargs$ is an attack relation and $\asupps\subseteq 2^\aargs\times \aargs$ a support relation.
\end{definition}
BSAFs generalise both bipolar AFs (BAFs)~\cite{DBLP:books/sp/09/CayrolL09} where each $(H,t)\in \aatts\cup \asupps$ satisfies $|H|=1$ and abstract argumentation frameworks (AFs) \cite{Dung95} where the same holds but $\asupps=\emptyset$.
For $t\in \aargs$, we let $\Att(t)=\{ \aset \mid (\aset,t)\in\aatts\}$
denote the set of all attackers $x$; then for a set $T\subseteq \aargs$, we let $\Att(T)=\bigcup_{t\in T} \Att(t)$; and
analogously for $\Sup(t)$.%

ABAFs are closely related to abstract argumentation~\cite{CyrasFST2018,DBLP:conf/kr/BertholdR024}.
We recall two graph-based representations of ABA. 

First, we recall the BSAF representation of ABAFs. 
Here, assumptions are the nodes in the graph and the derivations $S\vdash p$ induce hyperedges that encode attacks and supports.
\begin{definition}
    Let $D=(\lit,\asm,\rules,\contraryempty)$ be an ABAF.
    By $\BF_D=(\mathcal{A},\aatts_D,\asupps_D)$, we denote the BSAF corresponding to $D$ where
     $\aatts_D=\Att(\mathcal{A})$ and $\asupps_D=\Sup(\mathcal{A})$ 
    denote the set of all attacks and supports in $D$, respectively.
\end{definition}
Second, we recall the AF instantiation.
Here, the abstract arguments correspond to the ABA arguments and the attacks between the arguments is determined by their claims and assumptions. Note that the instantiation is only defined for 
\emph{flat} ABAFs where assumptions cannot be derived, i.e., there is no argument $\aset\vdash_R p$ with $R\neq\emptyset$ and $p\in\asm$.
\begin{definition}\label{def:AF instantiation}
    Let $D=(\lit,\asm,\rules,\contraryempty)$ be a flat ABAF.
    By $\AF_D=(\allargs_D,\AFatts_{\allargs_D})$, we denote the AF corresponding to $D$ where
    $\AFatts_{\allargs_D}=\{(x,y)\mid \cl(x)\in \contrary{\asms(y)}\}$.
\end{definition}
\begin{example}\label{exm:bsaf}
Consider an ABAF 
$D$ with
assumptions $a$, $b$, $c$, $d$, their contraries $\contraryc{a}$, $\contraryc{b}$, $\contrary{c}$, and $\contrary{d}$, respectively, and rules
\begin{align*}
r_1= c\la a,b&&r_2=\contraryc{b}&\la a &r_3=\contraryc{c}&\la b,d 
\end{align*}
Note that $D$ is not flat since the assumption $c$ is in the head of $r_1$. 
Let $D'$ denote the ABAF without rule $r_1$.

Below, we depict the BSAF $F_D$ with two attacks 
\textcolor{violet}{$(\{b,d\},c)$} and 
$\textcolor{red}{(\{a\},b)}$ and a support $\textcolor{cyan}{(\{a,b\},e)}$;
the BSAF $F_{D'}$; and the AF $\AF_{D'}$ with assumption arguments $a,b,c,d$ and two rule-based arguments $x_1=\{a\}\vdash \contrary{b}$ and $x_2=\{b,d\}\vdash \contrary{c}$.
\begin{center}
		\begin{tikzpicture}[scale=0.8,>=stealth]
        \node at (-0.8,1.5) {$\BF_D$:};
		\path
        (0.2,1.3)node[arg] (b){$b$}
		(1,0)node[arg] (c){$c$}
		(0,0)node[arg] (a){$a$}
  		(1.2,1.3)node[arg] (d){$d$}
        ;
        
\path[thick,->,red]
(a)edge[out=160,in=205](b)
;
\path[thick,->,violet]
(b)edge[out=-25,in=90](c)
(d)edge[out=-115,in=90](c)
;

\path[thick,->,cyan,dashed]
(a)edge[out=35,in=150](c)
(b)edge[out=-95,in=150](c)
;
\begin{scope}
    [xshift=3.3cm]
        \node at (-0.8,1.5) {$\BF_{D'}$:};
		\path
        (0.2,1.3)node[arg] (b){$b$}
		(1,0)node[arg] (c){$c$}
		(0,0)node[arg] (a){$a$}
  		(1.2,1.3)node[arg] (d){$d$}
        ;
        
\path[thick,->,red]
(a)edge[out=160,in=205](b)
;
\path[thick,->,violet]
(b)edge[out=-25,in=90](c)
(d)edge[out=-115,in=90](c)
;

\end{scope}

\begin{scope}
    [xshift=6.6cm]
          \node at (-0.8,1.5) {$\AF_{D'}$:};
		\path
        (0,1)node[arg] (b){$b$}
		(1.8,0)node[arg] (c){$c$}
		(0,0)node[arg] (a){$a$}
  		(1.8,1)node[arg] (d){$d$}
  		(0.9,1.6)node[arg] (x){$x_1$}
  		(0.9,0.4)node[arg] (xx){$x_2$}
        ;
        
\path[thick,->]
(x)edge (b)
(x) edge (xx)
(xx) edge (c)
;

\end{scope}

\end{tikzpicture}
\end{center}

\end{example}

\subsection{Gradual Semantics}
Quantitative bipolar AFs (QBAFs)~\cite{BARONI2019252} generalise AFs and BAFs by introducing an intrinsic strength of arguments.\footnote{QBAFs are defined using a set with a preorder in \cite{BARONI2019252} but in this paper we use real numbers $\mathbb{R}$.}
\begin{definition}
    A quantitative BAF (QBAF) is a tuple $Q=(\aargs,\aatts,\asupps,\basescore)$ where $(\aargs,\aatts,\asupps)$ is a BAF and $\basescore:\aargs\rightarrow \mathbb{R}$ assigns each $x\in \aargs$ a base score.
\end{definition}
Gradual semantics for QBAFs can be defined in terms of an aggregation and influence function, allowing for modular combinations~\cite{DBLP:journals/corr/abs-1807-06685}.
The \emph{aggregation function} has the form 
$\alpha:\mathbb{R}^{m+k} \mapsto \mathbb{R}$ where the first $m$ variables correspond to the strength values of the attackers and the last $k$ variables to the strengths of the supporters.
The order of the attackers and supporters does not play a role; and different arguments may have the same base score.
To aid readability, we use multi-sets as input, i.e., sets of real-valued elements that allow for multiple instances for each of its elements, and write $\alpha(A,S)$ for multi-sets $A$ of size $m$ and $S$ of size $k$.
See Table~\ref{tab:aggregation functions} for examples of aggregation functions.
The \emph{influence function} has the form $\iota:\mathbb{R}^{2} \mapsto \mathbb{R}$. The first variable is the base score of an argument and the second is the aggregate obtained from $\alpha$, cf.\ Table~\ref{tab:influence functions}.
All considered aggregation and influence functions in Table~\ref{tab:aggregation functions} and Table~\ref{tab:influence functions} respectively satisfy important properties such as \emph{balance} and \emph{monotonicity}~\cite{BARONI2019252,DBLP:conf/atal/Potyka19}; definitions for which are in \ifappendix Appendix~\ref{appendix:prelims}. \else Appendix A in~\cite{arxivABAgrad}. \fi 

A gradual semantics is the result of the iterative application of a given aggregation and influence function.
\begin{definition}
    Let $Q=(\aargs,\aatts,\asupps,\basescore)$ be a QBAF with $|\aargs|=n$,
    let $\alpha$ be an aggregation and $\iota$ be an influence function.
    The \emph{strength evolution process} is a function $s:\mathbb{R}^n\rightarrow \mathbb{R}^n$ defined as follows.
    For all $a\in \aargs$, 
    $t\in \mathbb{N}$,
    \begin{align*}
        s(0)_a=&\ \basescore(a)\\
        s(t+1)_a= &\ \iota\left(\basescore(a),
        \alpha(A_t^a,S_t^a\right) )
    \end{align*}
    where $A_t^a = \{s(t)_b\mid b\in \Att(a)\}$ and $S_t^a=\{s(t)_b\mid b\in \Sup(a)\}$ denote the strengths of the attackers and supporters of $a$ at time $t$, respectively.

    The modular $(\alpha,\iota)$-semantics $\sigma$
    is defined by the limit  $$\sigma(a)=\lim_{t\rightarrow \infty}s(t)_a.$$
\end{definition}
Multiple gradual semantics have been studied in the literature, e.g., \cite{DBLP:journals/corr/abs-1807-06685,DBLP:conf/ecsqaru/AmgoudB17}. We focus on DF-QuAD~\cite{DBLP:conf/kr/RagoTAB16} that uses $\alpha_\Pi$ and $\iota^1_{lin}$ and quadratic energy (QE)~\cite{DBLP:conf/kr/Potyka18} that combines $\alpha_\Sigma$ and $\iota^1_q$.
\begin{table}[t]
    \centering
        \caption{Influence functions}
    \label{tab:influence functions}
    \begin{tabular}{ c c}
       $\iota^k_{lin}(b,w)=b\!+\!\frac{b}{k}\!\cdot\! \min\{0,w\}\!+\!\frac{1\!-\!b}{k}\!\cdot\! \max\{0,w\}$ & \textbf{(lin$_k$)}\\
       $\iota^k_{q}(b,w)=b\!-\!b\!\cdot\! h(\text{-}\frac{w}{k}) +b\!\cdot\! h(\frac{w}{k})$ & \textbf{(QE$_k$)}\\
\qquad where $h(w)=\frac{\max\{0,w\}^2}{(1+\max\{0,w\})^2}$
\end{tabular}
\end{table}
\begin{table}[t]
    \centering
        \caption{Aggregation functions }
    \label{tab:aggregation functions}
    \begin{tabular}{l  r}
         $\alpha_\Sigma(A,S)=\Sigma_{s\in S}s-\Sigma_{a\in A}a$ & \textbf{(Sum)}  \\
         $\alpha_\Pi(A,S) = \Pi_{a\in A}(1-a)-\Pi_{s\in S}(1-s)$  & \textbf{(Prod)}
    \end{tabular}
\end{table}
\tikzstyle{workflow}=[draw,very thick, rectangle, rounded corners, minimum height=0.7cm, minimum width=1cm, 
inner sep=5pt,align=center]
\begin{figure}[b]
    \centering
    \begin{tikzpicture}
        \path
        node[workflow] at (0,0) (aba){{\footnotesize ABAF $D$}}
        node[workflow,label={below:{\scriptsize 1. Abstraction}}] at (2.2,0) (qbaf){{\footnotesize BSAF $\BF_D$}}
        node[workflow,label={below:{\scriptsize 2. Att/Sup strength}}] at (4.4,0) (arg){{\footnotesize $\setag(\aset)$}}
        node[workflow,label={above:{\scriptsize 3. Aggregate}}] at (6.6,.6) (sem){{\footnotesize $\alpha(A,S)$}}
        node[workflow,label={below:{\scriptsize 4. Influence}}] at (6.6,-.6) (asm)
        {{\footnotesize $\iota(b,a)$}}
        ;

        \path[draw,very thick,->,>=stealth]
        (aba) edge (qbaf)
        (qbaf) edge (arg);
        \draw[very thick,->,>=stealth,dash pattern=on 3pt off 1.5pt] (arg)  -- (sem);
        \draw[very thick,->,>=stealth,dash pattern=on 3pt off 1.5pt] (sem)  -- (asm);
        \draw[very thick,->,>=stealth,dash pattern=on 3pt off 1.5pt] (asm)  -- (arg);
        
    \end{tikzpicture}
    \caption{
    Gradual ABA semantics:
    Given ABAF $D$, 1.\ compute the  BSAF $F_D$, 
    2.\ the attacker's and supporter's strength, 3.\ aggregate, and 4.\ apply the influence function to update the assumption's strength. Repeat Steps 2 to 4 until convergence (dashed lines).}
    \label{fig:best route}
\end{figure}
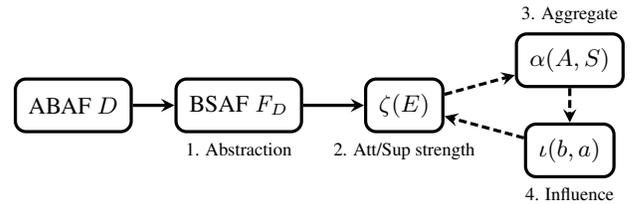
\section{Gradual ABA Semantics}\label{sec:grad ABA}
We introduce gradual semantics for ABA. Thereby, we will utilise BSAFs as they are a syntactical abstraction of ABAFs that concisely represents the interactions between assumptions; all other elements that contribute only indirectly are neglected in the abstraction.

First, let us define quantitative ABAFs.
\begin{definition}
    A quantitative ABAF is a tuple $(\lit,\asm,\rules,\contraryempty,\basescore)$ where  $(\lit,\asm,\rules,\contraryempty)$ is an ABAF and $\basescore:\mathcal{A}\rightarrow \mathbb{R}$ assigns a base score to each $a\in\mathcal{A}$. 
\end{definition}
In this work, we focus on functions $\basescore:\mathcal{A}\rightarrow [0,1]$.
Throughout the paper, we drop `quantitative' and simply write ABAF. 

We introduce \textit{modular} gradual ABA semantics, providing the flexibility to adapt the semantics to different use cases and reuse well-established methods from QBAFs. Our proposed gradual ABA semantics are based on the following procedure, cf.~Figure~\ref{fig:best route}. 
For an ABAF $D\!=\!(\lit,\asm,\rules,\contraryempty,\basescore)$, 
\begin{enumerate}
    \item first, compute the attacks and supports in $D$ to obtain the BSAF representation $\BF_D=(\mathcal{A},\aatts_D,\asupps_D)$;
    \item compute the strength of each set-attack (-support) by applying a \emph{set aggregation function} $\setag$ to the multi-set $C$ corresponding to the attacking (supporting, respectively) set;
    \item for each assumption $a\in \mathcal{A}$, use an aggregation function $\alpha$ to update the strength of $a$, based on the multi-sets $A$ and $S$ corresponding to the strengths of the attackers $\Att(a)$ and supporters $\Sup(a)$ of $a$, computed in Step 2;
    \item apply the influence function $\iota(b,w)$, with $b=\basescore(a)$ the base score of $a$ and $w=\alpha(A,S)$ the result from Step~3,
    \item repeat Steps 2-4, using the updated base scores of the assumptions computed in Step 4. 
\end{enumerate}
We will first identify suitable set aggregation functions $\setag(S)$ in Section~\ref{subsec:aggregate sets BSAF} and
discuss their interplay with aggregation functions in Section~\ref{sec: aggregating sets atts sups}.
We define gradual ABA semantics and analyse their behaviour in Section~\ref{subsec:bsaf semantics}. 

\newcommand{\Maxx}{\textit{Max}}
\subsection{Set Aggregation Functions}\label{subsec:aggregate sets BSAF}
We now identify suitable set-aggregation functions to compute the strength of the attacking and supporting sets.
\begin{definition}
    A set aggregation function has the form $\setag:\mathbb{R}^d\rightarrow \mathbb{R}$ whose 
    parameters 
    correspond to the strength values of an (attacking or supporting) set of assumptions.
    By $\maxval_\setag$, we denote the maximal value in $\setag(\mathbb{R}^d)$.
\end{definition}
Since the order of the assumptions should not play a role, we consider multi-sets $S$ as input of $\setag$ and simply write $\setag(S)$.
We stipulate $\setag(S)\geq 0$ for all $S$; moreover, for most set aggregation functions that we consider, it holds that $\maxval_\setag=1$.

What behaviour do we expect from a set aggregation function $\setag$?
Let us sketch our expectations in an example. 
\begin{example}\label{ex:gradual ABA ex}
Consider an assumption $a$ which is attacked by three sets. For simplicity, we state the attacking sets in terms of their weights:
$A_1=\{0.1,0.2,0.9\}$, $A_2=\{0.9\}$,
and $A_3=\{0,0.4\}$.
We discuss some observations. 
\begin{itemize}
    \item The weight of $A_2$ depends only on a single value. It thus makes intuitively sense to set $\setag(A_2)=0.9$. 
    \item Comparing $A_1$ and $A_2$, we observe that $A_2$ requires fewer assumptions to attack $a$, and all weights in $A_1$ are smaller or equal to the weight in $A_2$. Intuitively, $A_2$ should thus be stronger than $A_1$: the fewer assumptions needed the more effective the attack (similar in spirit to Occam's Razor).
    \item Among the given attacks, $A_3$ contains the lowest weight (weight $0$). 
    Intuitively, an attack should not be stronger than each weakest link (similar in spirit to~\cite{Spaans2021}); thus, $A_3$ is the weakest attack in this setting.
\end{itemize}
\end{example}
Below, we formalise desired principles.
For a set aggregation function $\setag$, we let $\nonmax{S}=\{w\in S\mid w\neq \maxval_\setag\}$ .
\begin{definition}\label{def:setag desiderata}
A set aggregation function $\setag$ satisfies
\begin{description}
     \item[Occam's Razor] iff $S'\subseteq S$ implies $\setag(S)\leq \setag(S')$;\hfill \ocr
     \item[Factuality] iff $\setag(\emptyset)=\maxval_\setag$; \hfill \fact
     \item[Neutrality] iff $\setag(S)=\setag(\nonmax{S})$; \hfill \neut
     \item[Identity] iff $\setag(\{s\})= s $; \hfill \ide
     \item[Weakest link limiting] iff $\setag(S)\leq \min S\!\cup\! \{1\}$; \hfill \wl
    \item[Void] iff
    $\setag(S)=0$ whenever $0\in S$. \hfill \void
 \end{description}
\end{definition}
Void and weakest link limiting appear in the context of argument strength aggregation~\cite{Spaans2021}, formalising that an argument cannot be stronger than its weakest link.
Identity appears also in~\cite{Spaans2021,DBLP:journals/argcom/RossitMDM21}.

\begin{table}[t]
     \centering
\caption{Set aggregation functions}\label{tab:set-aggregation functions}
    \begin{tabular}{l l}
          $\setag_\Sigma(S)=\Sigma_{s\in S}s$ \textbf{(Sum)} &
          $\setag_\bot(S) = \min S\!\cup\! \{1\}$  \textbf{(Min)}\vspace{2pt} \\
         $\setag_\Pi(S) = \Pi_{s\in S} s $  \textbf{(Prod)} & 
         $\setag_\top(S) = \max S\!\cup\! \{0\}$  \textbf{(Max)}   
    \end{tabular}
\end{table}
We study four basic aggregation methods: sum $\setag_\Sigma$, product $\setag_\Pi$, minimum $\setag_\bot$, and maximum $\setag_\top$.
Intuitively, product $\setag_\Pi$ and minimum $\setag_\bot$ are more cautious, they resemble conjunctive aggregation methods, while sum $\setag_\Sigma$ and maximum $\setag_\top$ follow a more credulous approach by adopting disjunctive aggregation methods.
 The functions are stated in Table~\ref{tab:set-aggregation functions}.
We examine our functions based on the above properties. While $\setag_\Pi$ and $\setag_\bot$ satisfy all desired properties, our findings show that $\setag_\Sigma$ and $\setag_\top$ are less suited to the purpose.

\begin{restatable}{proposition}{PropSetagSatisfiesProps}\label{prop:SatDesiderata setag}
    $\setag_\Pi$ and $\setag_\bot$ satisfy all considered properties;
    $\setag_\Sigma$ and $\setag_\top$ satisfy \ide\ only. 
\end{restatable}

We aim to introduce notions of monotonicity and balance
for the composition of aggregation and set aggregation functions,
similar in spirit to monotonicity and balance for influence and aggregation functions~\cite{DBLP:conf/atal/Potyka19,PotykaB24}.
Towards a formalisation, we introduce technical notions that allow us to compare multi-sets.\footnote{A discussion on why dominance, used to compare multi-sets in~\cite{PotykaB24}, is not suitable here is in \ifappendix Appendix~\ref{appendix:prelims}\else Appendix A~\cite{arxivABAgrad}\fi.}
Intuitively, a multi-set $A$ is \emph{superior} to $S$ if $A$ is smaller 
and 
wins the pairwise comparison of their elements.
In the comparison, 
 we disregard elements with maximal strength, we treat them as facts, similar to \citeauthor{prakken2018probabilistic}~\shortcite{prakken2018probabilistic} in the context of probabilistic argumentation.  
This guarantees that superiority is compatible with our desiderata in Definition~\ref{def:setag desiderata}.%
\begin{definition}\label{def:superiority}
    Let $A,S$ be two multi-sets. 
    $A$ is superior to
    $S$, in symbols $A\geqsup S$, iff one of the following holds:
    \begin{enumerate}
        \item $\nonmax{A}=\nonmax{S}=\emptyset$; or
        \item 
        for each $S'\subseteq \nonmax{S}$, $|S'|=|\nonmax{A}|$, there is a bijective function $f:\nonmax{A}\rightarrow S'$, such that for each $w\in \nonmax{A}$, it holds that $w \geq f(w)$.
    \end{enumerate}
\end{definition}
We  can simplify the definition slightly, as we show next.
Below, we let $\Maxx_k(S)$ 
denote the multi-set that contains the $k$ largest element of a multi-set $S$.
\begin{restatable}{lemma}{LemmaEasyDefSuper}
    Let $A,S$ be two multi-sets. 
    $A\geqsup S$ iff 
    $\nonmax{A}=\nonmax{S}=\emptyset$; 
    or
    there is a bijection $f:\nonmax{A}\rightarrow \Maxx_k(\nonmax{S})$, $k=\vert \nonmax{A}\vert$, such that for each $w\in \nonmax{A}$, it holds that $w \geq f(w)$.
\end{restatable}
\begin{example}
Consider $A=\{0.2,0.2,0.45,0.95,1\}$
and $S=\{0.001,0.01,0.2,0.4,0.9,1,1,1,1\}$. The following bijection verifies that $A$ is superior to $S$, as it maps all elements in $\nonmax{A}=\{0.2,0.2,0.45,0.95\}$
onto elements in  $\Maxx_4(\nonmax{S})=\{0.01,0.2,0.4,0.9\}$ so that $w\geq f(w)$:

\begin{center}
\begin{tabular}{l|c c c c c c c c c c c c c}
         $\nonmax{A}$ & 0.95 & 0.45 & 0.2 & 0.2 \\
         $f(\nonmax{A})$ & 0.9 & 0.4 & 0.2 & 0.01 
\end{tabular}    
\end{center}
\end{example}
Next, we introduce \emph{$\nonmaxm$-equivalence} and show that superiority and $\nonmaxm$-equivalence are closely related: if two multi-sets are mutually superior then they are $\nonmaxm$-equivalent.
\begin{definition}
    Two multi-sets $A,S$ are \emph{$\nonmaxm$-equivalent}, in symbols $A \eqsup S$, iff $\nonmax{A}=\nonmax{S}$.
\end{definition}
\begin{restatable}{proposition}{propBalanceSup}
\label{prop:balance and geqsup}
    ${A} \eqsup {S}$ iff $A\geqsup S$ and ${S}\geqsup{A}$.
\end{restatable}

\subsection{Aggregating Set-Attacks and -Supports}\label{sec: aggregating sets atts sups}
We investigate the interplay between set aggregation and aggregation functions.
Thereby, our objects of interest are \emph{multi-sets of multi-sets} 
corresponding to sets of attackers and supporters, denoted by $\multiA$ and $\multiS$, respectively.

We define dominance and balance below. 
Intuitively, $\multiA$ dominates $\multiS$
if for each $S\in \multiS$, there is some (distinct) superior set $A\in \multiA$.
Thereby, we ignore sets that contain the number $0$ (by the weakest link limiting property, we can disregard them). 
Below, we let 
$\Pos(\multiA)=\{A\in \multiA\mid 0\notin A\}$.
\begin{definition}
    Let $\multiA,\multiS$ be sets of multi-sets. 
    \begin{description}
        \item[Dominance]     
        $\multiA$ dominates
        $\multiS$, in symbols $\multiA\geqset \multiS$, 
    iff  \hfill \dom
    \begin{enumerate}
        \item $\Pos(\multiA)=\Pos(\multiS)=\emptyset$;
        or
        \item there is $\multiA'\subseteq \Pos(\multiA)$ and 
    a bijection $f:\multiA' \rightarrow \Pos(\multiS)$ 
    s.t.\ $A\geqsup f(A)$ for all $A\in \multiA'$.
    \end{enumerate} 
    \item[Balance] $\multiA$ and $\multiS$ are balanced, in symbols $\multiA\eqset \multiS$, iff 
    $\multiA\geqset \multiS$ and $\multiS\geqset\multiA$.
    \hfill \bal
    \end{description}
\end{definition}
In contrast to superiority, dominance requires that the dominating set $\multiA$ contains at least as many elements as the dominated set $\multiS$ does (disregarding sets that contain $0$).

For a set of sets $\multiA$, we write $\setag(\multiA)=\{\setag(A)\mid A\in\multiA\}$.

\begin{definition}
    (Set) aggregation functions $\alpha$ and $\setag$ satisfy $(\alpha,\setag)$-monotonicity iff
    \begin{itemize}
        \item $\multiA \geqset \multiS\ \Rightarrow\  \alpha(
        \setag(\multiA),
        \setag(\multiS)
        )\leq 0$
        \item $\multiS\geqset \multiA\ \Rightarrow\  \alpha(
        \setag(\multiA),
        \setag(\multiS)
        )\geq 0$
        \item $\multiA\geqset \multiA'\ \Rightarrow\ \alpha(
        \setag(\multiA),
        \setag(\multiS)
        )\leq \alpha(
        \setag(\multiA'),
        \setag(\multiS)
        )$
        \item $\multiS\geqset \multiS'\ \Rightarrow\  \alpha(
        \setag(\multiA),
        \setag(\multiS)
        )\leq \alpha(
        \setag(\multiA),
        \setag(\multiS')
        )$
    \end{itemize}
\end{definition}

\begin{definition}
A set aggregation function $\setag$ satisfies
\begin{description}
    \item[$\succeq$-preservation] iff
    $\multiA\geqset \multiS\ \Rightarrow\ \setag(\multiA)\geqsup \setag(\multiS)$.
    \hfill \spres
\end{description}
\end{definition}

\begin{restatable}{proposition}{PropSatisfySpres}\label{prop:satisfy spres}
    $\setag_\Pi$ and $\setag_\bot$ satisfy $\spres$.
\end{restatable}
Note that $\setag_\Sigma$ and $\setag_{\top}$ violate $\spres$. 
For $\setag_\Sigma$, this can be seen, for instance, by taking $A=\{0.2\}$ and $S=\{0.1,0.1,0.1\}$. Then $\{A\}\geqset \{S\}$ since $A\geqsup S$. However, $\setag_\Sigma(A)=0.2$ and $\setag_\Sigma(S)=0.3$.

We give conditions under which aggregation and set aggregation functions satisfy the combined monotonicity.%
\begin{restatable}{proposition}{PropSatisfactionSETAGALPHAmon}
Set aggregation function $\setag$ and aggregation function $\alpha$ satisfy $(\alpha,\setag)$-monotonicity whenever $\alpha$ satisfies monotonicity and
$\setag$ satisfies $\succeq$-preservation.
\end{restatable}

\subsection{Introducing Gradual ABA Semantics}\label{subsec:bsaf semantics}
We are now ready to define and study gradual ABA semantics.
As depicted in Figure~\ref{fig:best route}, a key part of the gradual evaluation of an ABAF $D=(\lit,\asm,\rules,\contraryempty)$ involves its BSAF representation $\BF_D=(\mathcal{A},\aatts_D,\asupps_D)$. 
In this section, we will thus introduce gradual BSAF semantics first and define gradual ABAF semantics based on the BSAF semantics.

\begin{definition}
    A quantitative BSAF is a tuple $F=(\aargs,\aatts,\asupps,\basescore)$ where $(\aargs,\aatts,\asupps)$ is a BSAF and $\basescore:\aargs\rightarrow \mathbb{R}$ is a function that assigns a base score to each $a\in\aargs$.
\end{definition}
In this paper, we focus on functions $\basescore:\mathcal{A}\rightarrow [0,1]$.
We drop `quantitative' and simply write BSAF.

\begin{definition}
    Let $F=(\aargs,\aatts,\asupps,\basescore)$  be a BSAF with $|\aargs|=n$,
    let $\setag$ be a set aggregation function, 
    $\alpha$ an aggregation function, and $\iota$ an influence function. 
    The \emph{strength evolution process} is a function $s:\mathbb{R}^n\rightarrow \mathbb{R}^n$ defined as follows.
    For all $a\in \aargs$, 
    $t\in \mathbb{N}$,
    \begin{align*}
        s(0)_a=&\ \basescore(a)\\
        s(t+1)_a= &\ \iota\left(\basescore(a),
        \alpha(A_t^a,S_t^a\right)
    \end{align*}
    where $A_t^a = \setag(\{s(t)_B\mid B\in \Att(a)\})$ and $S_t^a=\setag(\{s(t)_B\mid B\in \Sup(a)\})$ are multi-sets containing the aggregated strengths of the attackers and supporters of $a$ at time $t$, respectively.

    The modular $(\setag, \alpha,\iota)$-semantics $\sigma_\BF$
    is defined by the limit  $$\sigma_\BF(a)=\lim_{t\rightarrow \infty}s(t)_a.$$
\end{definition}
We say a semantics $\sigma$ is well-defined if $\sigma(a)$ converges (admits a unique fixed point) for each $a\in \aargs$. 

Below, we define gradual ABA semantics.
\begin{definition}
    Let $D=(\lit,\asm,\rules,\contraryempty,\basescore)$ be an ABAF and $F_D$ be the corresponding BSAF. We define $\sigma_{D}:=\sigma_{\BF_D}$.
\end{definition}
We drop the subscripts $\BF$, $D$, $\BF_D$ when it is clear from the context.
In the remainder of the paper, we identify ABAF and BSAF semantics. All results hold for both semantics.
\begin{remark}
\label{re:assiciate ABAF and BSAF}
In the remainder of this section, we associate an ABAF $D=(\lit,\asm,\rules,\contraryempty,\basescore)$ with the corresponding BSAF $\BF_D=(\mathcal{A},\aatts_D,\asupps_D,\basescore)$. 
To avoid confusion between ABA arguments and BSAF arguments, we write ``assumptions'' to refer to the (formally abstract) arguments in $\BF_D$. 
\end{remark}

The triple $(\setag, \alpha,\iota)$ which uniquely defines a modular semantics $\sigma$ is called the kernel of $\sigma$. Below, we adapt the notion of an elementary kernel definition by Potyka~(\citeyear{DBLP:conf/atal/Potyka19}).
\begin{definition}
    A kernel is a triple $(\alpha,\iota,\setag)$ consisting of an aggregation, influence, and a set-aggregation function. 
    The kernel is elementary iff all functions are Lipschitz-continuous, satisfy monotonicity and balance, $\alpha$ satisfies neutrality and $\setag$ satisfies void and $\succeq$-preservation.\footnote{Neutrality, monotonicity, and balance for aggregation and influence functions are defined in \ifappendix Appendix~\ref{appendix:prelims}; \else Appendix A; \fi Lipschitz-continuity is defined in \ifappendix Appendix~\ref{appendix:convergence}. \else Appendix C~\cite{arxivABAgrad}. \fi}
\end{definition}

\newcommand{\indep}{\textbf{(IND)}}
\newcommand{\anon}{\textbf{(A)}}
\newcommand{\dir}{\textbf{(DIR)}}
\newcommand{\IAM}{\textbf{(IAM)}}
\newcommand{\ISM}{\textbf{(ISM)}}
\newcommand{\RM}{\textbf{(RM)}}
\newcommand{\IB}{\textbf{(IB)}}
\newcommand{\RB}{\textbf{(RB)}}
\newcommand{\dual}{\textbf{(D)}}
\newcommand{\open}{\textbf{(O)}}

We study the central properties of gradual semantics~\cite{PotykaB24}
in the context of gradual ABA semantics.
Below, we define set-operations component-wise, e.g., $\BF\cup \BF'=(\aargs\cup \aargs',\aatts\cup \aatts',\asupps\cup \asupps',\basescore\cup \basescore')$ for two BSAFs $\BF=(\aargs,\aatts,\asupps,\basescore)$,  
$\BF'=(\aargs',\aatts',\asupps',\basescore')$. 

\begin{definition}
Let $\BF=(\aargs,\aatts,\asupps,\basescore)$,  
$\BF'=(\aargs',\aatts',\asupps',\basescore')$ be two BSAFs. Provided $\sigma_F$ and $\sigma_{F'}$ are well-defined, a gradual BSAF semantics $\sigma$ satisfies,
\begin{description}
     \item[Anonymity] iff
     $\sigma_F(a)=\sigma_{F'}(f(a))$
     for every edge- and label-preserving graph isomorphism $f:F\rightarrow F'$;
     \hfill \anon
     \item[Independence] iff $F\cap F'=\emptyset$ implies $\sigma_{F\cup F'}$ is well-defined and $\sigma_F(a)=\sigma_{F\cup F'}(a)$ for all $a\in A$;
     \hfill \indep
     \item[Directionality] iff $A=A'$, $\basescore=\basescore'$, and $A'=A\cup \{(T,h)\}$, then for all $c\in A$, so that there is no directed path from any $t\in T$ to $c$, $\sigma_F(c)=\sigma_{F'}(c)$.
     \hfill \dir

     \item[Individual A-Monotonicity] iff $\Att(a)\geqset \Sup(a)$ implies $\sigma(a)\leq \basescore(a)$;
     \hfill \IAM     
     \item[Individual S-Monotonicity] iff $\Sup(a)\geqset \Att(a)$ implies $\sigma(a)\geq \basescore(a)$;
     \hfill \ISM
     \item[Relative Monotonicity] iff $\basescore(a)\!\leq\! \basescore(b)$, $\Att(a)\!\geqset\! \Att(b)$, and $\Sup(b)\geqset \Sup(a)$ implies $\sigma(a)\leq \sigma(b)$;
     \hfill \RM
     \item[Individual Balance] iff $\Att(a)\eqset \Sup(a)$ implies $\sigma(a)=\basescore(a)$;
     \hfill \IB
     \item[Relative Balance] iff $\basescore(a) \!=\! \basescore(b)$, $\Att(a)\!\eqset\! \Att(b)$, and $\Sup(b)\!\eqset\! \Sup(a)$ implies $\sigma(a) = \sigma(b)$;
     \hfill \RB
 \end{description}
\end{definition}
We show that anonymity, independence and directionality hold for all modular semantics for BSAFs. For semantics with elementary modular kernel, all remaining properties are satisfied. The theorem below summarises our observations. 

\begin{restatable}{theorem}{ThmProperties}
    Each modular semantics satisfies  \indep, \anon, \dir. Each modular semantics with an elementary kernel additionally satisfies $\IAM$, \ISM, \RM, \IB, \RB.%
\end{restatable}
\begin{proof}[Proof (sketch).]
We present the proof of \IAM; the remaining proofs are similar and can be found in the appendix.
 Let $\sigma$ denote a semantics with an elementary kernel and let $A_a=\setag(\Att(a))$ and $S_a=\setag(\Sup(a))$.
Then, since $\sigma$ is a fixed point of the strength evolution process, $$\sigma(a)=\iota(\basescore(a),
        \alpha(A_a,S_a)).$$
        By assumption, $\Att(a)\geqset \Sup(a)$.
        Since $\alpha$ is monotone and $\setag$ satisfies $\succeq$-preservation, it holds that 
        they satisfy $(\alpha,\setag)$-monotonicity. We thus obtain $\alpha(A_a,S_a)\leq 0$. Hence, $\iota(\basescore(a), \alpha(A_a,S_a)) \leq \basescore(a)$ by monotonicity and balance of $\iota$.
        We obtain $\sigma(a)\leq \basescore(a)$.\qedhere
\end{proof}

QBAF semantics are a special case of BSAF semantics, in the following sense. 
\begin{restatable}{proposition}{propQBAFtoBSAF}\label{prop:same as QBAF kernek}
    Let $\BF=(\aargs,\aatts,\asupps,\basescore)$ be a BSAF with $|\aset|=1$ for all $(\aset,a)\in \aatts\cup \asupps$. 
    Let $\sigma$ be a modular semantics with kernel $(\setag,\alpha,\iota)$ where $\setag$ satisfies \ide. Then, 
     $\sigma_\BF$ corresponds to a QBAF semantics with kernel $(\alpha,\iota)$.
\end{restatable}

We investigate the convergence of modular semantics with elementary kernels. 
First, if the BSAF is acyclic then convergence is guaranteed.
We define cycles with respect to the \emph{primal graph}~\cite{DvorakKUW24} 
where each hyperedge $(\aset,a)$ is interpreted as simple edges $(e,a)$ for each $e\in \aset$.
A BSAF is acyclic if its corresponding primal graph is. 
\begin{restatable}{proposition}{propConvergeneacyclic}
    Let $\BF=(\aargs,\aatts,\asupps,\basescore)$ be an acyclic BSAF.
    Then each elementary modular semantics converges.
\end{restatable}
We establish convergence guarantees for the general case.
A function $f$ converges with exponential speed to a fixed point $x^*$ if $\vert f^n(x)-x^*\vert \leq cd^n$ for some $c\!\geq\! 0$ and $0\!\leq\! d\!<\!1$.
\begin{restatable}
{proposition}{convergenceP}
Let $\BF=(\aargs,\aatts,\asupps,\basescore)$ be a BSAF; let $(\setag,\alpha,\iota)$ be an elementary kernel with 
Lipschitz-constants $\lambda_\setag$, $\lambda_\alpha$, and $\lambda_\iota$, respectively. 
Let $\lambda_D=\lambda_\setag\cdot \lambda_\alpha\cdot \lambda_\iota$.
If $\lambda_D< 1$ then the strength evolution process $s_D$ converges with exponential speed to a fixed point $x^*$.
\end{restatable}
We give convergence guarantees for our proposed semantics.
Below, we write $\aatts(a)=\{(\aset,a)\mid E\subseteq \aargs\}$ to denote the set of all attackers of $a$; $\asupps(a)$ for supporters.
\begin{restatable}
    {proposition}{convergenceOurSems}
Let $\BF=(\aargs,\aatts,\asupps,\basescore)$ be a BSAF; 
let $d=\max \{|\aatts(a)\cup \asupps(a)| \mid a\in\aargs\}$ denote the maximum indegree of assumptions in $\BF$ and $h=\max \{|\aset|\mid (\aset,a)\in\aatts\cup \asupps\}$ the maximum size of an attacking or supporting set. 
     Let $\alpha_X$ for $X\in \{\Pi,\Sigma\}$, $\iota^k_Y$ for $Y\in \{lin,q\}$. 
     The strength evolution process $s_D$ is guaranteed to converge for
     \begin{itemize}
         \item for $\setag_{\Sigma}$, $Z\in\{\Pi,\Sigma\}$, whenever $hd < k$;
         \item for $\setag_{Z}$, $Z\in\{\top,\bot\}$, whenever $d < k$.
     \end{itemize}
\end{restatable}
For $k=2$, convergence is guaranteed for $\setag_\top$ and $\setag_\bot$
if each element has at most one attacker or supporter.
If, in addition, the BSAF corresponds to a QBAF (all attacks and supports have size $1$), the result also holds for $\setag_\Sigma$ and $\setag_\Pi$.

\begin{example}
    Let us head back to the ABAF $D$ from our introductory Example ~\ref{exm:intro}. We evaluate $a$ wrt.\ $\setag_\Pi$ and $\setag_\bot$ in combination with the DF-QuAD semantics ($\alpha_\Pi$ and $\iota_{lin}^1$).
    \begin{itemize}
        \item For $\basescore_1$, the strength of the attack $(\{b,c\},a)$ is $\setag_\Pi(\{1,1\})=\setag_\bot(\{1,1\})=1$. We obtain $\sigma^1_D(a)=0$.
        \item For $\basescore_2$, we have $\setag_\Pi(\{0.1,0.2\})\!=\!0.02$ and $\setag_\bot(\{0.1,0.2\})\!=\!0.1$. Thus, $(\setag_\Pi,\alpha_\Pi,\iota_{lin}^1)$ yields $\sigma^2_D(a)=0.98$ and $(\setag_\bot,\alpha_\Pi,\iota_{lin}^1)$ yields $\sigma^3_D(a)=0.9$.
    \end{itemize}
    As anticipated, the difference in the base scores is reflected in the final strength. So, if rain ($b$) and the absence of friends ($c$) in unlikely, then the strength of $a$ does not decrease much.
\end{example}

\newcommand{\argbase}{\beta}
\tikzstyle{workflow}=[draw,very thick, rectangle, rounded corners, minimum height=0.7cm, minimum width=1cm, 
inner sep=5pt,align=center]
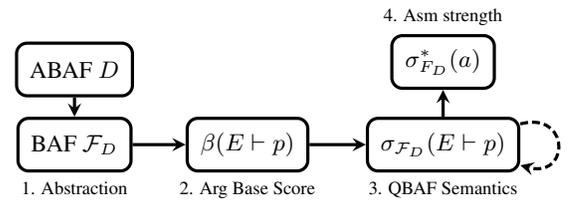
\begin{figure}[b]
    \centering
    \begin{tikzpicture}
        \path
        node[workflow] at (2.2,1) (aba){{\footnotesize ABAF $D$}}
        node[workflow,label={below:{\scriptsize 1. Abstraction}}] at (2.2,0) (qbaf){{\footnotesize BAF $\AF_D$}}
        node[workflow,label={below:{\scriptsize 2. Arg Base Score}}] at (4.5,0) (arg){{\footnotesize $\argbase(\aset \vdash p)$}}
        node[workflow,label={below:{\scriptsize 3. QBAF Semantics}}] at (7.1,0) (sem){{\footnotesize $\sigma_{\AF_D}(\aset\vdash p)$}}
        node[workflow,label={above:{\scriptsize 4. Asm strength}}] at (7.1,1.1) (asm)
        {{\footnotesize $\sigma^*_{F_D}(a)$}}
        ;

        \path[draw,very thick,->,>=stealth]
        (aba) edge (qbaf)
        (qbaf) edge (arg)
        (arg) edge (sem)
        (sem) edge (asm)
        (sem) edge[loop right, looseness=3.2, dash pattern=on 3pt off 1.5pt] (sem)
        ;
    \end{tikzpicture}
    \caption{
    QBAF baseline: Given ABAF $D$ and base score function $\basescore$, 1.\ compute the corresponding BAF $\AF_D$, 2.\ the argument's base score, 3.\ apply QBAF semantics until convergence (dashed loop), and 4.\ compute assumption strength. 
    }
    \label{fig:AF route}
\end{figure}
\section{A QBAF Baseline for ABA}\label{sec:instantiation}
In the realm of extension-based semantics, the direct (assumption-based) semantics for ABA are closely related to the (argument-based) semantics for AFs. 
It is well-known that the semantics correspond to each other~\cite{DBLP:journals/ai/BondarenkoDKT97}.
We set out to investigate if a similar correspondence can be established for gradual ABA as well.
To do so, we use argument base score functions together with the AF representation to instantiate a given ABAF as a QBAF. In this way, we can use QBAF semantics to evaluate an ABAF with weighted assumptions (see Figure~\ref{fig:AF route}).
Given an ABAF $D$ with base scores $\basescore$ and a QBAF semantics~$\sigma$,
\begin{enumerate}
    \item compute the abstract representation QBAF $\AF_D$;
    \item compute the argument's base score $\argbase(S\vdash p)$, given  $\basescore$;
    \item apply the QBAF semantics and compute $\sigma_{F_D}(S\vdash p)$;
    \item compute the strength of the assumptions, $\sigma^*_{F_D}(a)$.
\end{enumerate}
The standard instantiation is defined for flat ABAF only. Below, we extend Definition~\ref{def:AF instantiation} to the non-flat case.
\begin{definition}\label{def:new AF instantiation}
    Let $D=(\lit,\asm,\rules,\contraryempty)$ be an ABAF.
    By $\AF_D=(\allargs_D,\AFatts_{\allargs_D},\asupps_{\allargs_D})$, we denote the BAF corresponding to $D$ where
    $\AFatts_{\allargs_D}=\{(x,y)\mid \cl(x)\in \contrary{\asms(y)}\}$ and
    $\asupps_{\allargs_D}=\{(x,y)\mid \cl(x)\in \asms(y)\}$.
\end{definition}
While this instantiation does not preserve extension-based semantics~\cite{DBLP:conf/aaai/0001PRT24} it captures the syntactic attack and support relation between the ABA arguments. 

We discuss argument base score functions that extend the BAF instantiation to a QBAF in Section~\ref{subsec:arg strength}, while methods to extract the assumption strength are in Section~\ref{subsec:asm strength}.

\subsection{Base Score Function for ABA Arguments}\label{subsec:arg strength}

We now identify suitable base score functions.
We define a function $\argbase$ that assigns a base score to each argument $S\vdash p$.%
\begin{definition}
Let $D=(\lit,\asm,\rules,\contraryempty,\basescore)$ be an ABAF.
An \emph{argument base score function $\argbase:\allargs_D \rightarrow \mathbb{R}$} 
assigns a base score to each argument $S\vdash p$,
$S\subseteq \mathcal{A}$, $p\in\mathcal{L}$, $|S|=d$. 
\end{definition}
We consider \emph{syntax independence} (called anonymity in~\cite{Spaans2021} and related to rewriting in~\cite{Jedwabny_20}) which intuitively means that the base score of an argument is independent of its structure.%
\begin{definition} An argument base score function $\argbase$ satisfies
\begin{description} 
    \item[Syntax Independence] iff $\argbase(S\vdash_{t} p )=\argbase(S\vdash_{t'} p)$. \hfill\textbf{(SI)}
\end{description}
\end{definition}
We consider a class of argument base score functions that neglect the specific argument tree; the strength is entirely determined by the assumptions in the argument.
\begin{definition}
An assumption-based argument base score function $\argbase:\mathbb{R}^d \rightarrow \mathbb{R}$ 
assigns each argument $S\vdash_t p$, $|S|=d$, a base score based on the assumption's strengths.
\end{definition}
We focus on assumption-based functions. 
This choice has convenient implications: each such argument base score function satisfies syntax independence $(\textbf{SI})$ by design.

Most of the desirable properties from Definition~\ref{def:setag desiderata} in the context of set aggregation functions apply to the case of argument base score functions as well. 
Occam's razor, for instance, formalises that the fewer assumptions the stronger an argument.
As shown in Section~\ref{subsec:aggregate sets BSAF}, the product and minimum function satisfy all desired properties. 
We will thus consider the following argument base score functions. 
\begin{center}
    \begin{tabular}{r r}
          $\argbase_\Pi(\aset)\! =\! \prod_{e\in \aset} e $ \textbf{(Prod)} & 
          $\argbase_\bot(\aset)\! =\! \min \aset\!\cup\! \{1\}$   \textbf{(Min)} 
    \end{tabular}    
\end{center}

Both functions satisfy syntax independence by definition.%
 \begin{restatable}{proposition}{PropDesiderataArg}
     $\argbase_\Pi$ and $\argbase_\bot$ satisfy $\si$, $\ocr$, $\fact$, $\neut$, $\ide$, $\wl$, and $\void$.
 \end{restatable}

\newcommand{\asmstrength}{\delta}
\subsection{Assumption Strength Computation}
\label{subsec:asm strength}
We investigate suitable functions to determine the final strength of the assumptions, based on the final strength of the arguments that derive them. That is, to compute the strength of $a$, we take all arguments $S\vdash a$ into account. 
Since a single claim may be supported by several arguments, we have several choices to obtain the strength of the assumptions.%
\begin{definition}
Let $D=(\lit,\asm,\rules,\contraryempty,\basescore)$ be a ABAF, $\sigma$ be a QBAF semantics, and let $S\subseteq \allargs_D$ denote the set that contains all 
arguments with claim $a$, i.e., $S=\{E\vdash a\mid E\subseteq \mathcal{A}\}$, let $x=\{a\}\vdash a$ for each $a\in\asm$. We define 

\begin{center}
    \begin{tabular}{r r}
          $\sigma^*_{\mathrm{asm}}(S)=\sigma(x_a)$ \textbf{(Asm)} & 
          $\sigma^*_\bot(S) = \min S $  \textbf{(Min)} \vspace{2pt}\\
         $\sigma^*_{\mathrm{avg}}(S) = \frac{1}{|S|}\sum S$   \textbf{(Avg)} &
         $\sigma^*_\top(S) = \max S$ \textbf{(Max)}  \vspace{2pt}
    \end{tabular}    
\end{center}

\end{definition}

Note $\sigma^*_{avg}$ averages the outcome while the other functions represent choices to pick a representative strength. 

For flat ABAFs, all of the options coincide (for assumptions) since we have only one argument for each assumption.%
\begin{restatable}{proposition}{PropSameArg}
    Let $D$ be a flat ABAF and $\sigma$ a QBAF semantics. Then 
    $\sigma=\sigma'$ for all $\sigma,\sigma~\in\{\sigma^*_{\mathrm{asm}},\sigma^*_{\bot},\sigma^*_{\top},\sigma^*_{\mathrm{avg}}\}$.
\end{restatable}
We note that convergence depends on the choice of the QBAF semantics, i.e., 
$\sigma^*$ converges iff $\sigma$ does.      

We investigate the relation of the QBAF baseline to gradual ABA semantics. 
Note that the topology of a BSAF representation will typically differ from the BAF instantiation (see, for instance, Example~\ref{exm:bsaf}).
We obtain equality in the special case in which no rules are given.  

\begin{restatable}{proposition}{PropCorrespAnrg}
Let $D=(\lit,\asm,\emptyset,\contraryempty,\basescore)$ be an ABAF, 
 $F_D$ denote the BSAF corresponding to $D$,
$\AF_D$ denote the QBAF corresponding to $D$,
$\sigma_D$ be a modular $(\setag,\alpha,\iota)$-semantics for $D$ and $\sigma_{\AF_D}$ denote the corresponding modular $(\alpha,\iota)$-semantics.
It holds that $\sigma_D=\sigma_{F_D}=\sigma_{\AF_D}$. 
\end{restatable}

\section{Experimental Evaluation}
\label{sec:experiments}

We now explore the differences between our proposed gradual ABA (BSAF) and the presented QBAF baseline (BAF) and within various modules of the semantics.
Our experiments show 
that BSAFs converge 90\% of the scenarios while BAFs do so only in around 70\% of the scenarios; moreover, BSAFs converge faster (in around 30 iterations) than BAFs do (in around 45 iterations).%

\begin{figure}[t]
    \centering
    \includegraphics[width=\linewidth]{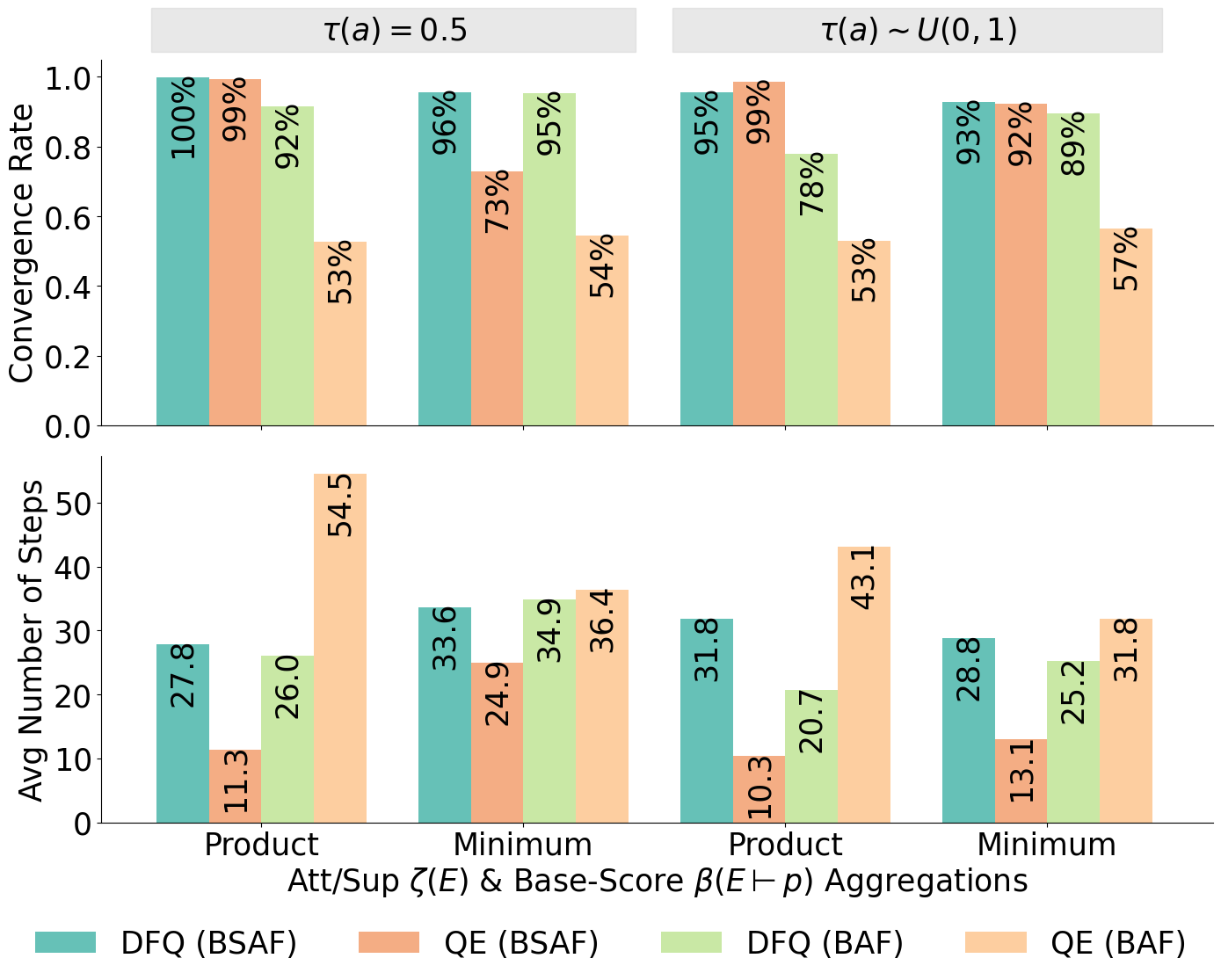}
    \caption{Global convergence rates (top, y-axis) and average steps to converge (bottom, y-axis) for the BSAF and BAF approach.  We compare two modular semantics (DFQ vs.\ QE), two $\Att$/$\Sup$ and Base-score aggregations (x-axis, resp. for BSAF and BAF), across two base‐score initialisations ($\tau=0.5$ vs.\ $\tau\sim U(0,1)$).}
    \label{fig:overall_convergence}
\end{figure}

\begin{figure}[ht]
    \centering
    \includegraphics[width=\linewidth]{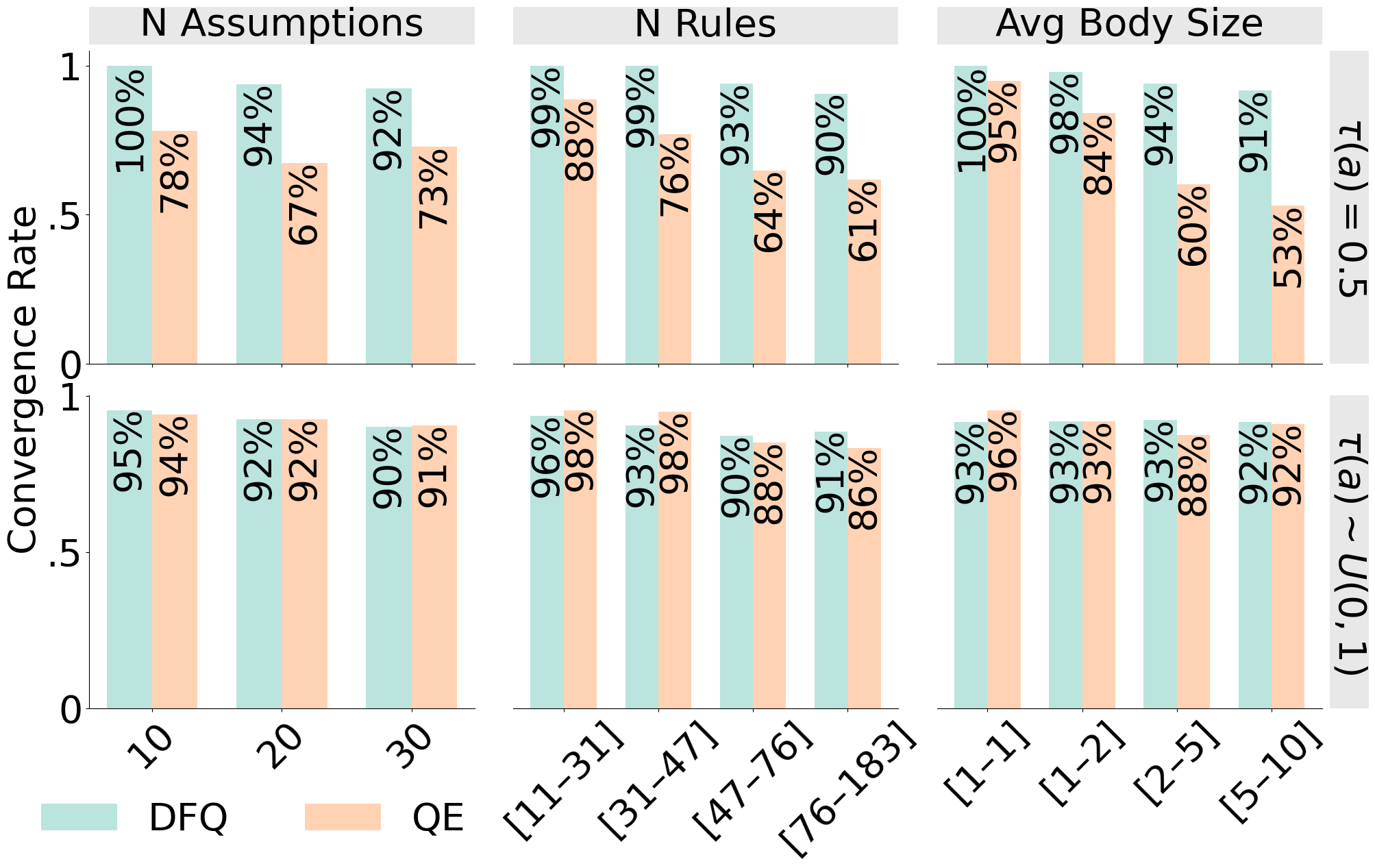}
    \caption{Sensitivity of the gradual ABA semantics (BSAF) to ABAF structural characteristics.  Convergence rates are binned by number of assumptions (left), number of rules (center), and average rule body size (right), for both DFQ and QE under constant ($\tau=0.5$, top row) and random ($\tau\sim U(0,1)$, bottom row) base‐scores. The set aggregation is minimum (product in \ifappendix Appendix~\ref{app:additional_results}). \else Appendix E.3). \fi   } 
    \label{fig:sensitivity}
\end{figure}
\subsubsection{Experimental Setup.}
We evaluate
BSAF and BAF
on a suite of $1440$ randomly generated ABAFs (960 non-flat, 480 flat), following the procedure in~\cite{DBLP:conf/ijcai/LehtonenRT0W24}. Details are provided in \ifappendix Appendix~\ref{sec:data-generation}. \else Appendix E.1~\cite{arxivABAgrad}. \fi

The runtime of both BSAF and BAF approaches depends on the instantiations (ABAF to SETAF and ABAF to AF, respectively), which can be both exponential in the number of assumptions in the worst case, 
and on the number of passes, which depends on the structure of the resulting abstract instance. As discussed in the previous sections, convergence cannot always be guaranteed; therefore, we constrained the number of iterations. 
Each run has a 10-minute time budget and up to 5000 iterations; we record 92 timeouts during abstraction (all on non-flat ABAFs), with 10 additional timeouts for BSAF and 89 for BAF in the semantics phases (see Figures~\ref{fig:best route} and \ref{fig:AF route}). 

\subsubsection{Convergence Metrics.}
We use two metrics to quantify the 
convergence behaviour of our gradual ABA semantics:
\begin{itemize}
  \item \emph{Global convergence rate:} proportion of ABAFs in which every assumption strength settles (within $\epsilon=10^{-3}$) over the final $\delta=5$ updates before hitting the cap on the number of iterations.
  \item \emph{Average steps to converge:} mean iteration step to convergence, computed only over the runs that do converge.
\end{itemize}
Further methodological details are in \ifappendix Appendix~\ref{app:metrics_details}. \else Appendix E.2~\cite{arxivABAgrad}. \fi

\subsubsection{Modular Components Combinations.}
We test two types for each of the modules presented in the previous sections: 
\begin{itemize}
  \item Aggregations \& Influence Functions: 
    \begin{itemize}
      \item \emph{DF-QuAD} (DFQ): $\alpha_\Pi$ with linear influence $\iota^1_{\mathrm{lin}}$  
      \item \emph{Quadratic Energy} (QE): $\alpha_{\Sigma}$ with quadratic influence $\iota^1_{q}$
    \end{itemize}
  \item Sets Aggregation \& Base‐score Functions:
    \begin{itemize}
      \item \emph{BSAF:} Att/Sup aggregation $\setag\in\{\setag_\bot,\setag_\Pi\}.$  
      \item \emph{BAF:} Base score aggregation $\argbase\in\{\argbase_\bot,\argbase_\Pi\}.$
    \end{itemize}
\end{itemize}

We use two base score initialisation: constant ($\tau(a)=0.5$) and random ($\tau(a)\sim U(0,1)$) and two assumptions strengths functions $\sigma^*\in\{\sigma^*_{\mathrm{avg}},\sigma^*_{\mathrm{asm}}\}$ (for BAF only). Results for the latter ablation, with no significant differences observed, are in \ifappendix Appendix~\ref{app:additional_results}. \else Appendix E.3~\cite{arxivABAgrad}. \fi

\paragraph{BSAF vs.\ BAF.}
We show global convergence rates (top) and average steps to converge (bottom) in Figure~\ref{fig:overall_convergence}. Overall:
\begin{enumerate}
  \item \textbf{Convergence Rate:}  BSAF achieves convergence in the range of 73-95\% of runs, compared to 53–95\% under BAF, with DFQ leading across both BSAFs and BAFs.
  \item \textbf{Convergence Speed:}  BSAF converges in 10–30 iterations, while BAF requires 20–55.
  \item \textbf{Sensitivity to Initialisation:}  BSAF is insensitive to $\tau$, while BAF shows wide variance across initialisations
  .
\end{enumerate}

\subsubsection{DFQ vs.\ QE within BSAF.}
DFQ consistently achieves $\ge$90\% convergence across all settings. QE converges fastest under Product aggregation (circa 10 steps for both $\tau=0.5$ and $\tau\sim U(0,1)$) but slows to 25–35 steps under Minimum. Interestingly, DFQ suffers from random initialisation (versus constant), while QE benefits, possibly reflecting the open-mindedness vs.\ conservativeness trade-off in \cite{PotykaB24}.

To explore the differences between DFQ and QE in more detail, we conduct an additional analysis on how convergence rates vary by structural ABAFs' features.
The summary of the results is shown in Figure~\ref{fig:sensitivity}. We split convergence rates by:
  \emph{number of assumptions} (left column);
  \emph{number of rules} (center column);
  \emph{average rule body size} (right column).
Results are given for constant ($\tau=0.5$, top row) and random ($\tau\sim U(0,1)$, bottom row) base scores initialisations. DFQ remains above 90\% in every bucket, whereas QE degrades on larger, more complex instances. More comparisons, including Product Att/Sup aggregation and the same breakdown for convergence speed and for the BAF approach, are in \ifappendix Appendix~\ref{app:additional_results}. \else Appendix E.3~\cite{arxivABAgrad}. \fi

\begin{figure}[t]
    \centering
    \includegraphics[width=\linewidth]{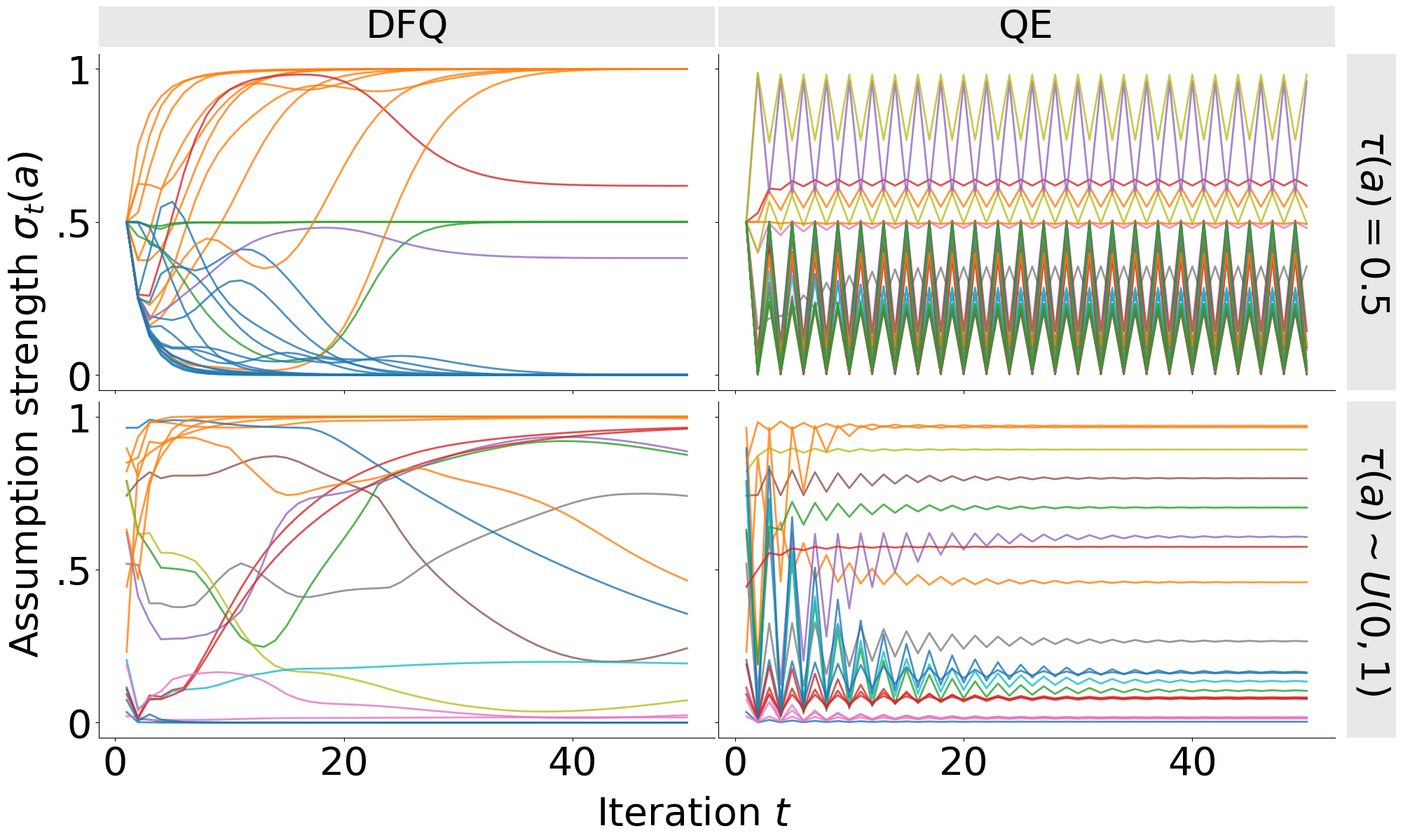}
    \caption{Sample assumption‐strength trajectories $\sigma_t(a)$ under the gradual ABA semantics.  DFQ (left column) versus QE (right column), with fixed base‐scores $\tau=0.5$ (top row) and random $\tau\sim U(0,1)$ (bottom row).  DFQ changes smoothly, whereas QE exhibits sustained oscillations or slow damping.}
    \label{fig:strength_trajectories}
\end{figure}
\subsubsection{Strength Trajectories.}
We show sample assumption-strength trajectories $\sigma_t(a)$ under BSAF in Figure~\ref{fig:strength_trajectories}. In the first example ABAF (30 assumptions, $\tau=0.5$, $\setag_{\Pi}$, top row)
we show that DFQ smoothly converges, while QE oscillates indefinitely.
  Another example is shown in the bottom row
  (20 assumptions, $\tau\sim U(0,1)$, $\setag_{\bot}$) where we observe that QE quickly settles, while DFQ never settles (not even with 5000 iterations, see \ifappendix Appendix~\ref{app:additional_results}). \else Appendix E.3~\cite{arxivABAgrad}). \fi

Overall, BSAF outperforms BAF in both convergence reliability and speed. Within BSAF, DFQ provides the greatest robustness across ABAFs' characteristics and initialisations, while QE trades some resilience for rapid convergence.

\section{Conclusion}

We have presented a family of novel gradual semantics for ABA, filling a gap in the literature and bridging to the extensive body of work on gradual semantics for abstract forms of argumentation. 
Our gradual semantics is defined via BSAFs that allow us to focus on the (supporting and attacking) relations between the assumptions. 
We have formally and experimentally evaluated  our gradual ABA semantics against a QBAF-based baseline for ABA and have shown that it exhibits desirable characteristics. 

\subsubsection{Future work.}
By presenting the very first purpose-built gradual semantics for ABA, our work opens several avenues for future work.
One regards the correspondence of gradual ABA semantics to other evaluation paradigms such as extension-based or labelling-based semantics. In this aspect, it would be interesting to investigate gradual interpretations of the \emph{closure property} known from extension-based semantics as the requirement that ensures that an extension contains all assumptions that it derives. 
Also, extending gradual semantics to ABA with preferences using set-to-set attacks~\cite{DimopoulosD0R0W24} would be an interesting avenue for future research.
In this work, we assume that the non-assumptions of our knowledge base (i.e., rules and sentences) 
do not have an intrinsic strength on their own. For future work, we want to investigate settings with varying intrinsic strength of sentences that are non-assumptions; and settings in which the strength of rules is part of the evaluation. 
Given that assumptions can render the rules defeasible, we can introduce rule weights using a dedicated assumption for each rule which represents its weight. In this way, the assumptions’ weight can be equated to the rules' weights.
Our argument base score functions could also be used to evaluate the strength of ordinary conclusions.
We also plan to explore concrete applications of our gradual semantics, e.g., for causal discovery with the ABA framework in \cite{russo2024argumentativecausaldiscovery}.

\section*{Acknowledgments}
Rapberger and Russo were funded by the ERC under the ERC-POC programme (grant number  101189053) while Rago and Toni under the EU’s Horizon 2020 research and innovation programme (grant number 101020934); Toni also by J.P. Morgan and by the Royal Academy of Engineering under the Research Chairs and Senior Research Fellowships scheme. 

\bibliographystyle{kr}

\begin{thebibliography}{}

\bibitem[\protect\citeauthoryear{Amgoud and Ben{-}Naim}{2015}]{Amgoud_15}
Amgoud, L., and Ben{-}Naim, J.
\newblock 2015.
\newblock Argumentation-based ranking logics.
\newblock In {\em AAMAS},  1511--1519.

\bibitem[\protect\citeauthoryear{Amgoud and
  Ben{-}Naim}{2017}]{DBLP:conf/ecsqaru/AmgoudB17}
Amgoud, L., and Ben{-}Naim, J.
\newblock 2017.
\newblock Evaluation of arguments in weighted bipolar graphs.
\newblock In {\em {ECSQARU}},  25--35.

\bibitem[\protect\citeauthoryear{Amgoud and
  Doder}{2018}]{DBLP:conf/kr/AmgoudD18a}
Amgoud, L., and Doder, D.
\newblock 2018.
\newblock Gradual semantics for weighted graphs: An unifying approach.
\newblock In {\em KR},  613--614.

\bibitem[\protect\citeauthoryear{Amgoud \bgroup et al\mbox.\egroup
  }{2008}]{DBLP:journals/ijis/AmgoudCLL08}
Amgoud, L.; Cayrol, C.; Lagasquie{-}Schiex, M.; and Livet, P.
\newblock 2008.
\newblock On bipolarity in argumentation frameworks.
\newblock {\em Int. J. Intell. Syst.} 23(10):1062--1093.

\bibitem[\protect\citeauthoryear{Ayoobi, Potyka, and Toni}{2023}]{Ayoobi_23}
Ayoobi, H.; Potyka, N.; and Toni, F.
\newblock 2023.
\newblock Sparx: Sparse argumentative explanations for neural networks.
\newblock In {\em {ECAI}},  149--156.

\bibitem[\protect\citeauthoryear{Baroni, Caminada, and
  Giacomin}{2018}]{BaroniCG18}
Baroni, P.; Caminada, M.; and Giacomin, M.
\newblock 2018.
\newblock Abstract argumentation frameworks and their semantics.
\newblock In {\em Handb. Formal Argument.} College Publications.
\newblock  159--236.

\bibitem[\protect\citeauthoryear{Baroni, Rago, and Toni}{2019}]{BARONI2019252}
Baroni, P.; Rago, A.; and Toni, F.
\newblock 2019.
\newblock From fine-grained properties to broad principles for gradual
  argumentation: {A} principled spectrum.
\newblock {\em Int. J. Approx. Reason.} 105:252--286.

\bibitem[\protect\citeauthoryear{Berthold, Rapberger, and
  Ulbricht}{2023}]{DBLP:conf/jelia/BertholdRU23}
Berthold, M.; Rapberger, A.; and Ulbricht, M.
\newblock 2023.
\newblock On the expressive power of assumption-based argumentation.
\newblock In {\em JELIA},  145--160.

\bibitem[\protect\citeauthoryear{Berthold, Rapberger, and
  Ulbricht}{2024}]{DBLP:conf/kr/BertholdR024}
Berthold, M.; Rapberger, A.; and Ulbricht, M.
\newblock 2024.
\newblock Capturing non-flat assumption-based argumentation with bipolar
  {SETAFs}.
\newblock In {\em KR},  128--133.

\bibitem[\protect\citeauthoryear{Besnard and
  Hunter}{2014}]{besnard2014constructing}
Besnard, P., and Hunter, A.
\newblock 2014.
\newblock Constructing argument graphs with deductive arguments: a tutorial.
\newblock {\em Argument Comput.} 5(1):5--30.

\bibitem[\protect\citeauthoryear{Bondarenko \bgroup et al\mbox.\egroup
  }{1997}]{DBLP:journals/ai/BondarenkoDKT97}
Bondarenko, A.; Dung, P.~M.; Kowalski, R.~A.; and Toni, F.
\newblock 1997.
\newblock An abstract, argumentation-theoretic approach to default reasoning.
\newblock {\em Artif. Intell.} 93:63--101.

\bibitem[\protect\citeauthoryear{Caminada and Gabbay}{2009}]{Caminada_09}
Caminada, M., and Gabbay, D.~M.
\newblock 2009.
\newblock A logical account of formal argumentation.
\newblock {\em Stud. Log.} 93(2-3):109--145.

\bibitem[\protect\citeauthoryear{Caminada and
  Schulz}{2017}]{DBLP:journals/jair/CaminadaS17}
Caminada, M., and Schulz, C.
\newblock 2017.
\newblock On the equivalence between assumption-based argumentation and logic
  programming.
\newblock {\em J. Artif. Intell. Res.} 60:779--825.

\bibitem[\protect\citeauthoryear{Cayrol and
  Lagasquie{-}Schiex}{2005}]{CayrolL05}
Cayrol, C., and Lagasquie{-}Schiex, M.
\newblock 2005.
\newblock Graduality in argumentation.
\newblock {\em J. Artif. Intell. Res.} 23:245--297.

\bibitem[\protect\citeauthoryear{Cayrol and
  Lagasquie{-}Schiex}{2009}]{DBLP:books/sp/09/CayrolL09}
Cayrol, C., and Lagasquie{-}Schiex, M.
\newblock 2009.
\newblock Bipolar abstract argumentation systems.
\newblock In {\em Argumentation in Artificial Intelligence}. Springer.
\newblock  65--84.

\bibitem[\protect\citeauthoryear{Cyras and Toni}{2016}]{CyrasT16}
Cyras, K., and Toni, F.
\newblock 2016.
\newblock {ABA+:} assumption-based argumentation with preferences.
\newblock In {\em KR},  553--556.

\bibitem[\protect\citeauthoryear{{\v C}yras \bgroup et al\mbox.\egroup
  }{2018}]{CyrasFST2018}
{\v C}yras, K.; Fan, X.; Schulz, C.; and Toni, F.
\newblock 2018.
\newblock Assumption-based argumentation: Disputes, explanations, preferences.
\newblock In {\em Handb. Formal Argument.} College Publications.
\newblock  365--408.

\bibitem[\protect\citeauthoryear{Cyras \bgroup et al\mbox.\egroup
  }{2021}]{DBLP:journals/argcom/CyrasOKT21}
Cyras, K.; Oliveira, T.; Karamlou, A.; and Toni, F.
\newblock 2021.
\newblock Assumption-based argumentation with preferences and goals for
  patient-centric reasoning with interacting clinical guidelines.
\newblock {\em Argument Comput.} 12(2):149--189.

\bibitem[\protect\citeauthoryear{Cyras, Heinrich, and
  Toni}{2021}]{DBLP:journals/ai/CyrasHT21}
Cyras, K.; Heinrich, Q.; and Toni, F.
\newblock 2021.
\newblock Computational complexity of flat and generic assumption-based
  argumentation, with and without probabilities.
\newblock {\em Artif. Intell.} 293:103449.

\bibitem[\protect\citeauthoryear{Dimopoulos \bgroup et al\mbox.\egroup
  }{2024}]{DimopoulosD0R0W24}
Dimopoulos, Y.; Dvor{\'{a}}k, W.; K{\"{o}}nig, M.; Rapberger, A.; Ulbricht, M.;
  and Woltran, S.
\newblock 2024.
\newblock Redefining {ABA+} semantics via abstract set-to-set attacks.
\newblock In {\em AAAI},  10493--10500.

\bibitem[\protect\citeauthoryear{Dung and Thang}{2010}]{DungT10}
Dung, P.~M., and Thang, P.~M.
\newblock 2010.
\newblock Towards (probabilistic) argumentation for jury-based dispute
  resolution.
\newblock In {\em COMMA},  171--182.

\bibitem[\protect\citeauthoryear{Dung}{1995}]{Dung95}
Dung, P.~M.
\newblock 1995.
\newblock On the acceptability of arguments and its fundamental role in
  nonmonotonic reasoning, logic programming and n-person games.
\newblock {\em Artif. Intell.} 77(2):321--357.

\bibitem[\protect\citeauthoryear{Dvor{\'{a}}k \bgroup et al\mbox.\egroup
  }{2024}]{DvorakKUW24}
Dvor{\'{a}}k, W.; K{\"{o}}nig, M.; Ulbricht, M.; and Woltran, S.
\newblock 2024.
\newblock Principles and their computational consequences for argumentation
  frameworks with collective attacks.
\newblock {\em J. Artif. Intell. Res.} 79:69--136.

\bibitem[\protect\citeauthoryear{Fan}{2023}]{Fan23}
Fan, X.
\newblock 2023.
\newblock Probabilistic deduction as a probabilistic extension of
  assumption-based argumentation.
\newblock In {\em AAMAS},  2352--2354.

\bibitem[\protect\citeauthoryear{Gonzalez \bgroup et al\mbox.\egroup
  }{2021}]{Gonzalez_21}
Gonzalez, M. G.~E.; Bud{\'{a}}n, M. C.~D.; Simari, G.~I.; and Simari, G.~R.
\newblock 2021.
\newblock Labeled bipolar argumentation frameworks.
\newblock {\em J. Artif. Intell. Res.} 70:1557--1636.

\bibitem[\protect\citeauthoryear{Hecham, Bisquert, and
  Croitoru}{2018}]{hecham2018flexible}
Hecham, A.; Bisquert, P.; and Croitoru, M.
\newblock 2018.
\newblock On a flexible representation for defeasible reasoning variants.
\newblock In {\em AAMAS},  1123--1131.

\bibitem[\protect\citeauthoryear{Heyninck, Raddaoui, and
  Stra{\ss}er}{2023}]{Heyninck_23}
Heyninck, J.; Raddaoui, B.; and Stra{\ss}er, C.
\newblock 2023.
\newblock Ranking-based argumentation semantics applied to logical
  argumentation.
\newblock In {\em {IJCAI}},  3268--3276.

\bibitem[\protect\citeauthoryear{Hunter}{2013}]{HUNTER201347}
Hunter, A.
\newblock 2013.
\newblock A probabilistic approach to modelling uncertain logical arguments.
\newblock {\em Int. J. Approx. Reason.} 54(1):47--81.

\bibitem[\protect\citeauthoryear{Hunter}{2022}]{hunter2022argument}
Hunter, A.
\newblock 2022.
\newblock Argument strength in probabilistic argumentation based on defeasible
  rules.
\newblock {\em Int. J. Approx. Reason.} 146:79--105.

\bibitem[\protect\citeauthoryear{Jedwabny, Croitoru, and
  Bisquert}{2020}]{Jedwabny_20}
Jedwabny, M.; Croitoru, M.; and Bisquert, P.
\newblock 2020.
\newblock Gradual semantics for logic-based bipolar graphs using t-(co)norms.
\newblock In {\em {ECAI}},  777--783.

\bibitem[\protect\citeauthoryear{Lehtonen \bgroup et al\mbox.\egroup
  }{2023}]{DBLP:conf/kr/LehtonenR0W23}
Lehtonen, T.; Rapberger, A.; Ulbricht, M.; and Wallner, J.~P.
\newblock 2023.
\newblock Argumentation frameworks induced by assumption-based argumentation:
  Relating size and complexity.
\newblock In {\em KR},  440--450.

\bibitem[\protect\citeauthoryear{Lehtonen \bgroup et al\mbox.\egroup
  }{2024}]{DBLP:conf/ijcai/LehtonenRT0W24}
Lehtonen, T.; Rapberger, A.; Toni, F.; Ulbricht, M.; and Wallner, J.~P.
\newblock 2024.
\newblock Instantiations and computational aspects of non-flat assumption-based
  argumentation.
\newblock In {\em IJCAI},  3457--3465.

\bibitem[\protect\citeauthoryear{Lehtonen, Wallner, and
  J{\"{a}}rvisalo}{2021a}]{DBLP:journals/jair/LehtonenWJ21}
Lehtonen, T.; Wallner, J.~P.; and J{\"{a}}rvisalo, M.
\newblock 2021a.
\newblock Declarative algorithms and complexity results for assumption-based
  argumentation.
\newblock {\em J. Artif. Intell. Res.} 71:265--318.

\bibitem[\protect\citeauthoryear{Lehtonen, Wallner, and
  J{\"{a}}rvisalo}{2021b}]{DBLP:journals/tplp/LehtonenWJ21}
Lehtonen, T.; Wallner, J.~P.; and J{\"{a}}rvisalo, M.
\newblock 2021b.
\newblock Harnessing incremental answer set solving for reasoning in
  assumption-based argumentation.
\newblock {\em Theory Pract. Log. Program.} 21(6):717--734.

\bibitem[\protect\citeauthoryear{Leite and Martins}{2011}]{Leite_11}
Leite, J., and Martins, J.~G.
\newblock 2011.
\newblock Social abstract argumentation.
\newblock In {\em {IJCAI}},  2287--2292.

\bibitem[\protect\citeauthoryear{Modgil and Prakken}{2014}]{Modgil_14}
Modgil, S., and Prakken, H.
\newblock 2014.
\newblock The \emph{ASPIC}\({}^{\mbox{+}}\) framework for structured
  argumentation: a tutorial.
\newblock {\em Argument Comput.} 5(1):31--62.

\bibitem[\protect\citeauthoryear{Mossakowski and
  Neuhaus}{2018}]{DBLP:journals/corr/abs-1807-06685}
Mossakowski, T., and Neuhaus, F.
\newblock 2018.
\newblock Modular semantics and characteristics for bipolar weighted
  argumentation graphs.
\newblock {\em CoRR} abs/1807.06685.

\bibitem[\protect\citeauthoryear{Oluokun \bgroup et al\mbox.\egroup
  }{2024}]{Oluokun_24}
Oluokun, B.; Paulino{-}Passos, G.; Rago, A.; and Toni, F.
\newblock 2024.
\newblock Predicting human judgement in online debates with argumentation.
\newblock In {\em CMNA 2024},  70--80.

\bibitem[\protect\citeauthoryear{Potyka and Booth}{2024}]{PotykaB24}
Potyka, N., and Booth, R.
\newblock 2024.
\newblock {Balancing Open-Mindedness and Conservativeness in Quantitative
  Bipolar Argumentation (and How to Prove Semantical from Functional
  Properties)}.
\newblock In {\em {KR}},  597--607.

\bibitem[\protect\citeauthoryear{Potyka}{2018}]{DBLP:conf/kr/Potyka18}
Potyka, N.
\newblock 2018.
\newblock Continuous dynamical systems for weighted bipolar argumentation.
\newblock In {\em {KR}},  148--157.

\bibitem[\protect\citeauthoryear{Potyka}{2019}]{DBLP:conf/atal/Potyka19}
Potyka, N.
\newblock 2019.
\newblock Extending modular semantics for bipolar weighted argumentation.
\newblock In {\em {AAMAS}},  1722--1730.

\bibitem[\protect\citeauthoryear{Prakken}{2018}]{prakken2018probabilistic}
Prakken, H.
\newblock 2018.
\newblock Probabilistic strength of arguments with structure.
\newblock In {\em {KR}},  158--167.

\bibitem[\protect\citeauthoryear{Prakken}{2024}]{PRAKKEN2024104193}
Prakken, H.
\newblock 2024.
\newblock An abstract and structured account of dialectical argument strength.
\newblock {\em Artif. Intell.} 335:104193.

\bibitem[\protect\citeauthoryear{Rago \bgroup et al\mbox.\egroup
  }{2016}]{DBLP:conf/kr/RagoTAB16}
Rago, A.; Toni, F.; Aurisicchio, M.; and Baroni, P.
\newblock 2016.
\newblock Discontinuity-free decision support with quantitative argumentation
  debates.
\newblock In {\em {KR}},  63--73.

\bibitem[\protect\citeauthoryear{Rago \bgroup et al\mbox.\egroup
  }{2024}]{Rago_25X}
Rago, A.; Vasileiou, S.~L.; Toni, F.; Son, T.~C.; and Yeoh, W.
\newblock 2024.
\newblock A methodology for incompleteness-tolerant and modular gradual
  semantics for argumentative statement graphs.
\newblock {\em CoRR} abs/2410.22209.

\bibitem[\protect\citeauthoryear{Rago \bgroup et al\mbox.\egroup
  }{2025}]{Rago_25}
Rago, A.; Cocarascu, O.; Oksanen, J.; and Toni, F.
\newblock 2025.
\newblock Argumentative review aggregation and dialogical explanations.
\newblock {\em Artif. Intell.} 340:104291.

\bibitem[\protect\citeauthoryear{Rapberger and
  Ulbricht}{2023}]{DBLP:journals/jair/RapbergerU23}
Rapberger, A., and Ulbricht, M.
\newblock 2023.
\newblock On dynamics in structured argumentation formalisms.
\newblock {\em J. Artif. Intell. Res.} 77:563--643.

\bibitem[\protect\citeauthoryear{Rossit \bgroup et al\mbox.\egroup
  }{2021}]{DBLP:journals/argcom/RossitMDM21}
Rossit, J.; Mailly, J.; Dimopoulos, Y.; and Moraitis, P.
\newblock 2021.
\newblock United we stand: Accruals in strength-based argumentation.
\newblock {\em Argument Comput.} 12(1):87--113.

\bibitem[\protect\citeauthoryear{Russo, Rapberger, and
  Toni}{2024}]{russo2024argumentativecausaldiscovery}
Russo, F.; Rapberger, A.; and Toni, F.
\newblock 2024.
\newblock Argumentative causal discovery.
\newblock In {\em {KR}}.

\bibitem[\protect\citeauthoryear{Schulz and
  Toni}{2017}]{DBLP:journals/ijar/SchulzT17}
Schulz, C., and Toni, F.
\newblock 2017.
\newblock Labellings for assumption-based and abstract argumentation.
\newblock {\em Int. J. Approx. Reason.} 84:110--149.

\bibitem[\protect\citeauthoryear{Skiba, Thimm, and Wallner}{2023}]{Skiba_23}
Skiba, K.; Thimm, M.; and Wallner, J.~P.
\newblock 2023.
\newblock Ranking-based semantics for assumption-based argumentation.
\newblock In {\em {FCR}},  44--52.

\bibitem[\protect\citeauthoryear{Skiba, Thimm, and
  Wallner}{2024}]{DBLP:conf/ratio/SkibaTW24}
Skiba, K.; Thimm, M.; and Wallner, J.~P.
\newblock 2024.
\newblock Ranking transition-based medical recommendations using
  assumption-based argumentation.
\newblock In {\em {RATIO}},  202--220.

\bibitem[\protect\citeauthoryear{Spaans}{2021}]{Spaans2021}
Spaans, J.~P.
\newblock 2021.
\newblock Intrinsic argument strength in structured argumentation: {A}
  principled approach.
\newblock In {\em {CLAR}},  377--396.

\bibitem[\protect\citeauthoryear{Toni}{2013}]{Toni13}
Toni, F.
\newblock 2013.
\newblock A generalised framework for dispute derivations in assumption-based
  argumentation.
\newblock {\em Artif. Intell.} 195:1--43.

\bibitem[\protect\citeauthoryear{Ulbricht \bgroup et al\mbox.\egroup
  }{2024}]{DBLP:conf/aaai/0001PRT24}
Ulbricht, M.; Potyka, N.; Rapberger, A.; and Toni, F.
\newblock 2024.
\newblock Non-flat {ABA} is an instance of bipolar argumentation.
\newblock In {\em {AAAI}},  10723--10731.

\bibitem[\protect\citeauthoryear{Zeng \bgroup et al\mbox.\egroup
  }{2020}]{DBLP:conf/atal/ZengSCLWCM20}
Zeng, Z.; Shen, Z.; Chin, J.~J.; Leung, C.; Wang, Y.; Chi, Y.; and Miao, C.
\newblock 2020.
\newblock Explainable and contextual preferences based decision making with
  assumption-based argumentation for diagnostics and prognostics of alzheimer's
  disease.
\newblock In {\em {AAMAS}},  2071--2073.

\end{thebibliography}

\ifappendix \appendix \else \end{document} \fi
\appendix

\clearpage        

\section{On Dominance and Balance (or: Extended Preliminaries)}\label{appendix:prelims}
In this section, we recall dominance for multi-sets as well as balance and monotonicity properties for aggregation and influence functions following ~\cite{PotykaB24}.
We furthermore discuss why dominance between multi-sets does not capture the desired comparison between multi-sets for our setting very well.

Below, we give the definition of dominance. Recall that for a multi-set $S$, we let $\Pos(S)=\{x\in S\mid x\neq 0\}$.
\begin{definition}\label{def:dominance multiset 1}
    Let $A,S$ denote two multi-sets.
    \begin{description}
        \item[Dominance] $A$ dominates $S$, in symbols $A\geqels S$, iff $\Pos(A)=\Pos(S)=\emptyset$ or there is $A'\subseteq \Pos(A)$ and a bijection $f:A' \rightarrow \Pos(S)$ s.t.\ $x\geq f(x)$ for all $x\in A'$;
        \item[Balance] $A$ and $S$ are balanced, in symbols $A\eqels S$, iff $\Pos(A)=\Pos(S)$.
    \end{description}
\end{definition}
As expected and as shown by \citeauthor{PotykaB24}~(\citeyear{PotykaB24}), $A\eqels S$ iff $A\geqels S$ and $A\geqels S$.

We observe that dominance,  cf.\ Definition~\ref{def:dominance multiset 1}, does not capture the desired relation between multi-sets in the context of set-aggregation of attackers (supporters) very well.
\begin{example}
    Consider again our two attacks $A_1=\{0.1,0.2,0.9\}$ and $A_2=\{0.9\}$,
    According to our desiderata in Definition~\ref{def:setag desiderata}, 
    the attack $A_2$ is stronger than $A_1$ since $A_1\supsetneq A_2$.
    However, it holds that $A_1\geqels A_2$.
\end{example}
In the following example, we discuss the difference between superiority and dominance of multi-sets.  
\begin{example}
Let $A=\{0.2,0.2,0.45,0.95,1\}$
and $S=\{0.01,0.2,0.4,0.9,1,1,1,1\}$.
It holds that $S$ dominates $T$, using the following mapping:
\begin{center}
\begin{tabular}{l|c c c c c c c c c c c c c}
         $S'$ & 1 & 1 & 1 & 1 & 0.9 \\
         $f(S')$ & 1 & 0.95 & 0.45 & 0.2 & 0.2 
\end{tabular}    
\end{center}
On the other hand, $A$ is superior over $S$.
In contrast to dominance, where we focus on the non-zero elements, we 
disregard the non-maximal elements when checking for superiority. Note $|\nonmax{A}|=4$.
It remains to check if we can find a bijection $f:\nonmax{A}\rightarrow Max_4(S)$ so that $x\geq f(x)$, $x\in A$. 
\begin{center}
\begin{tabular}{l|c c c c c c c c c c c c c}
         $A$ & 0.95 & 0.45 & 0.2 & 0.2 \\
         $f(A)$ & 0.9 & 0.4 & 0.2 & 0.01 
\end{tabular}    
\end{center}
\end{example}

In the setting of QBAFs, dominance plays an important role to define crucial properties. 
First, we recall properties of aggregation functions. 
\begin{definition}
An aggregation function $\alpha$ satisfies 
\begin{description}
\item[Neutrality] if $\alpha(A,S)=\alpha(\Pos(A),\Pos(S))$ 
\item[Monotonicity] if the following is satisfied: 
    \begin{itemize}
        \item $A\geqels S\ \Rightarrow\  \alpha(
        A,
        S
        )\leq 0$
        \item $S\geqels A\ \Rightarrow\  \alpha(
        A,
        S
        )\geq 0$
        \item $A\geqels A'\ \Rightarrow\ \alpha(
        A',
        S
        )\leq \alpha(
        A,
        S
        )$
        \item $S\geqels S'\ \Rightarrow\  \alpha(
        A,
        S
        )\leq \alpha(
        A,
        S'
        )$ 
    \end{itemize} 
    \item[Balance] if the following is satisfied:    
    \begin{itemize}
        \item $A\eqels S\ \Rightarrow\ \alpha(A,S)=0$
        \item $A\eqels A', S\eqels S'\ \Rightarrow\ \alpha(A,S)=\alpha(A',S')$
    \end{itemize}
\end{description}
\end{definition}
Note that balance is a consequence of monotonicity. 

Next, we recall properties of influence functions. 
\begin{definition}
An influence function $\iota$ satisfies 
\begin{description}
\item[Monotonicity] if the following is satisfied: 
    \begin{itemize}
        \item $w < 0\ \Rightarrow\  \iota(b,w)\leq b$
        \item $w > 0\ \Rightarrow\  \iota(b,w)\geq b$
        \item $w_1 \leq w_2\ \Rightarrow\  \iota(b,w_1)\leq \iota(b,w_2)$
        \item $b_1 \leq b_2\ \Rightarrow\  \iota(b_1,w)\leq \iota(b_2,w)$
    \end{itemize} 
    \item[Balance] if $\iota(b,0)=b$ 
\end{description}
\end{definition}

\section{Omitted proofs from Section~\ref{sec:grad ABA}}
In this section, we state the proofs about fundamental properties of gradual ABA semantics. 
\PropSetagSatisfiesProps*
\begin{proof}
Most properties follow directly from the definition. Note that weakest link limiting implies void.

$\setag_\Pi$: \ocr\ is satisfied since  $\basescore:\mathcal{A}\rightarrow[0,1]$ and $a\cdot b\leq a$ for all $a,b\in[0,1]$. 
$\fact$ is satisfied since the empty product is $1$.
To show $\neut$, first observe that $\maxval_\sigma=1$ since the base score is in the interval $[0,1]$. 
Thus, $$\setag_\Pi(S)= \prod_{w\in S}w = \prod_{w\in S, w\neq 1} w = \setag_\Pi(\nonmax{S}),$$
and therefore, $\neut$ is satisfied. 
\ide\ is trivially satisfied.
$\wl$ holds since $s,s'\leq 1$ for all $s,s'\in S$, thus, $\prod_{s\in S}s\leq \min S$ if $S\neq \emptyset$ and $\prod_{s\in S}=1$ otherwise. 

$\setag_\Sigma$: It is easy to see that most of the properties are violated, e.g., $a+b\geq a$ for all $a,b\in[0,1]$ thus $\ocr$ does not hold; the empty sum is $0\neq 1=\maxval_{\setag_\sigma}$ thus $\fact$ is violated; clearly, $a+1> a$, contradicting $\neut$; the sum over elements in $[0,1]$ is larger than its minimum, thus $\wl$ is violated; if $s\in S$ with $s> 0$ then $\void$ is violated.  

$\setag_{\top}$: Let $S=\{0.2,0.5\}$ and $S'=\{0.2\}$ then $\setag_\top(S)=0.5> 0.2=\setag_\top(S')$, violating $\ocr$;
$\max \emptyset\cup \{0\}=0\neq 1=\maxval_{\setag_\top}$ thus $\fact$ is violated; clearly, $\max \{0.2,1\}=1>0.2=\max\{0.2\}$, contradicting $\neut$; $\max S> \min S$ for $S$ defined as above; and if $s\in S$ with $s> 0$ then $\void$ is violated. 

$\setag_{\bot}$: $\min S \leq \min S'$ for each $S'\subseteq S$: if $\min S = s\in S'$ then  $\min S' = \min S $ otherwise $\min S' = s'$ for some $s'>s = \min S$, thus $\ocr$ is satisfied. 
$\fact$ holds by definition; 
 $\neut$ follows since $s\in [0,1]$ for all $s\in S$; thus either there is some $s\in S$ such that $s<1$ or $\nonmax{S}=\emptyset$ and thus $\min \nonmax{S}=1$ by definition of $\setag_{\bot}$;
 $\wl$ and $\void$ follow from the definition.
\end{proof}

\LemmaEasyDefSuper*
\begin{proof}
$(\Rightarrow)$ Let $A\geqsup S$. If 1.\ holds, we are done. Now assume 2.\ holds. Let $S'=\Maxx_k(\nonmax{S})$. Then there is a bijection $f:\nonmax{A}\rightarrow S'$ such that for each $x\in \nonmax{A}$, $w\geq f(w)$.

    ($\Leftarrow$)  If 1.\ holds, we are done. Otherwise, let $S'\subseteq S$ with size $\vert \nonmax{A}\vert = k= \vert S'\vert$.
    Then, for each $w\in S'$, we can find some distinct value in $\Maxx_k(\nonmax{S})$ that is greater or equal to $w$ (since $\Maxx_k(\nonmax{S})$ contains all the maximal elements in $\nonmax{S}$). That is, we can find a bijection $f:\Maxx_k(\nonmax{S})\rightarrow S'$ such that $w\geq f'(w)$ for all $w\in \Maxx_k(\nonmax{S})$.
    Thus, $f'\circ f: \nonmax{A}\rightarrow S'$ is a bijection that satisfies, for all $w\in \nonmax{A}$, $w\geq f(w)\geq f'(f(w))$, as desired.
\end{proof}

\propBalanceSup*
\begin{proof}
    ($\Rightarrow$) Let $f=id$.
    
    ($\Leftarrow$) $A\geqsup S$ implies that there is a bijection $f:\nonmax{A}\rightarrow \Maxx_k(\nonmax{S})$ thus $|\nonmax{A}|\leq |\nonmax{S}|$;
    ${S}\geqsup{A}$ implies that there is a bijection $f':\nonmax{S}\rightarrow Max_{k'}(\nonmax{A})$ thus $|\nonmax{S}|\leq |\nonmax{A}|\leq |\nonmax{S}|$. We obtain $|\nonmax{A}|=|\nonmax{S}|$ and $f,f'=id$. 
\end{proof}

We note that, in contrast to superiority and $\nonmaxm$-equivalence (and in contrast to the respective properties for aggregation functions),
it is not the case that dominance, i.e.,  $\multiA\geqset \multiS$ and $\multiS\geqset\multiA$, is equivalent to $\Pos(\multiA)=\Pos(\multiS)$.

We give an example below.
\begin{example}
    Let $\multiA=\{A\}$ and $\multiS=\{S\}$ with $A=\{1,0.1\}$ and $S=\{0.1\}$.
    Observe that $\nonmax{A}=\nonmax{S}$. 
    Then, $\multiA\geqset \multiS$ since $A\geqsup S$ and $\multiS\geqset\multiA$ since $S\geqsup A$.
    However, $\Pos(\multiA)=\multiA\neq \multiS=\Pos(\multiS)$.
\end{example}

\PropSatisfySpres*
\begin{proof}
    $\setag_\Pi$: 
    Suppose $\multiA\geqset \multiS$.
    First, suppose $\Pos(\multiA)=\Pos(\multiS)=\emptyset$, that is, each set in $\multiA$ and $\multiS$ contains $0$.
    Since $\setag_\Pi$ satisfies \void, it follows that $\setag_\Pi(A)=\setag_\Pi(S)=0$ for all $A\in \multiA$, $S\in \multiS$. Thus $\setag_\Pi(\multiA)\succeq \setag_\Pi(\multiS)$.

    Now, suppose there is $\multiA'\subseteq \Pos(\multiA)$ and 
    a bijection $f:\multiA' \rightarrow \Pos(\multiS)$ 
    s.t.\ $A\geqsup f(A)$ for all $A\in \multiA'$.

    Fix $A\in\multiA'$ and let $S=f(A)$.
    We show that $\setag_\Pi(A)\geq \setag_\Pi(S)$.
    Recall that
    $A\geqsup S$ iff 
    (i) $\nonmax{A}=\nonmax{S}=\emptyset$
    or
    (ii) there is a bijection $f:\nonmax{A}\rightarrow \Maxx_k(\nonmax{S})$, 
    such that for each $x\in \nonmax{A}$, it holds that $x \geq f(x)$.

    Suppose (i) holds.
    Since $\setag_\Pi$ satisfies \neut, we have $\setag_\Pi(A)=\setag_\Pi(\nonmax{A})=\setag_\Pi(\emptyset)$.
    Thus, $\setag_\Pi(A)=\setag_\Pi(S)$.

    Suppose (ii) holds. Let $s_1,\dots ,s_k=\Maxx_k(\nonmax{S})$ denote the $k=|\nonmax{A}|$ largest elements of $\nonmax{S}$, $s_i\neq 1$, let $s_i\leq s_{i+1}$. Let $a_1,\dots, a_k$ denote the corresponding elements in $\nonmax{A}$ so that $a_i\geq s_i$. Since all elements are in $[0,1)$, it holds that $\prod_{i\leq k}s_i\leq \prod_{i\leq k}a_i$. We obtain $\setag(A)\geq \setag(S)$.

    $\setag_\bot$: 
    Suppose $\multiA\geqset \multiS$.
    First, suppose $\Pos(\multiA)=\Pos(\multiS)=\emptyset$, that is, each set in $\multiA$ and $\multiS$ contains $0$.
    Since $\setag_\bot$ satisfies \void, it follows that $\setag_\bot(A)=\setag_\bot(S)=0$ for all $A\in \multiA$, $S\in \multiS$. Thus $\setag_\bot(\multiA)\succeq \setag_\bot(\multiS)$.

    Now, suppose there is $\multiA'\subseteq \Pos(\multiA)$ and 
    a bijection $f:\multiA' \rightarrow \Pos(\multiS)$ 
    s.t.\ $A\geqsup f(A)$ for all $A\in \multiA'$.

    Fix $A\in\multiA'$ and let $S=f(A)$.
    We show that $\setag_\bot(A)\geq \setag_\bot(S)$.
    Recall that
    $A\geqsup S$ iff 
    (i) $\nonmax{A}=\nonmax{S}=\emptyset$
    or
    (ii) there is a bijection $f:\nonmax{A}\rightarrow \Maxx_k(\nonmax{S})$, 
    such that for each $x\in \nonmax{A}$, it holds that $x \geq f(x)$.

    Suppose (i) holds.
    Since $\setag_\bot$ satisfies \neut, we have $\setag_\bot(A)=\setag_\bot(\nonmax{A})=\setag_\bot(\emptyset)$.
    Thus, $\setag_\bot(A)=\setag_\bot(S)$.

    Suppose (ii) holds. Let $s_1,\dots ,s_k=Max_k(S)$ denote the $k=|\nonmax{A}|$ largest elements of $S$, $s_i\neq 1$, let $a_1,\dots, a_k$ denote the corresponding elements in $\nonmax{A}$ so that $a_i\geq s_i$. Wlog, let $s_1$, $a_1$ denote the smallest element of the respective sets. Then $\min \Maxx_k(\nonmax{S})=s_1\leq a_1=\min \nonmax{A}$. Thus $\setag(A)\geq \setag(S)$.
\end{proof}
We provide counter-examples for     $\setag_\Sigma$ and   $\setag_{\top}$.
\begin{example}
    $\setag_\Sigma$ violates $\spres$: Let $A=\{0.2\}$ and $S=\{0.1,0.1,0.1\}$. Then $\{A\}\geqset \{S\}$ since $A\geqsup S$.
    However, $\setag_\Sigma(A)=0.2$ and $\setag_\Sigma(S)=0.3$.

    $\setag_{\top}$ violates $\spres$: Let $A=\{0,0.1\}$ and $S=\{0,0.2\}$. Then $\{A\}\geqset \{S\}$ but $\setag_{\bot}(A)=0.1>0.2=\setag_\bot(S)$.
    
\end{example}

\PropSatisfactionSETAGALPHAmon*
\begin{proof}
Since $\setag$ satisfies \spres, we have that every two multi-sets of mulit-sets $\multiA$, $\multiS$ can be aggregated into multi-sets $\setag(\multiA)$, $\setag(\multiS)$ in an order-preserving manner. Since $\alpha$ satisfies monotonicity, the statement follows. 
\end{proof}

\ThmProperties*
\begin{proof}
    Let $\sigma$ be a modular semantics. Then, $\sigma$ satisfies 
    \begin{description}
        \item[\anon] since the names of the assumptions does not influence their evaluation;
        \item[\indep] by definition, since the assumption's strength is influenced only by connected assumptions;
        \item[\dir] since, more precisely, only attack- or support-predecessors influence the strength of an assumption.
    \end{description}
    Let $A_a=\setag(\Att(a))$, $S_a=\setag(\Sup(a))$, $A_b=\setag(\Att(b))$, and $S_b=\setag(\Sup(b))$. 
    Let $\sigma$ have an elementary kernel. Then, $\sigma$ satisfies
        \begin{description}
        \item[\IAM] $\sigma$ is a fixed point of the strength evolution process, we have $$\sigma(a)=\iota(\basescore(a),
        \alpha(A_a,S_a)).$$
        By assumption, $\Att(a)\geqset \Sup(a)$.
        Since $\alpha$ is monotone and $\setag$ satisfies $\succeq$-preservation, it holds that 
        they satisfy $(\alpha,\setag)$-monotonicity. We thus obtain $\alpha(A_a,S_a)\leq 0$. Hence, $\iota(\basescore(a), \alpha(A_a,S_a) ) \leq \basescore(a)$ by monotonicity and balance of $\iota$.
        We obtain $\sigma(a)\leq \basescore(a)$.
        \item[\ISM] Analogous to \IAM.
        \item[\RM]  It holds that \begin{align*}
            \sigma(a)=&\ \iota(\basescore(a), 
        \alpha(A_a,S_a))\\
        \sigma(b)= &\ \iota(\basescore(b),
        \alpha(A_b,S_b))
        \end{align*}
        From $\Att(a)\geqset \Att(b)$ and by $(\alpha,\setag)$-monotonicity, we obtain $\alpha(A_a,S_a)\leq \alpha(A_b,S_a)$;
        further, from $\Sup(b)\geqset Sub(a)$, we get $\alpha(A_b,S_a)\leq \alpha(A_b,S_b)$.
        Thus, $\alpha(A_a,S_a)\leq \alpha(A_b,S_b)$. 
        Moreover, by assumption $\basescore(a)\leq \basescore(b)$ and by monotonicity of $\iota$, we obtain
        \begin{align*}
        \sigma(a)=\iota(\basescore(a), \alpha(A_a,S_a))\leq &\ \iota(\basescore(a),
        \alpha(A_b,S_b))\\
        \leq &\ \iota(\basescore(b),
        \alpha(A_b,S_b))=\sigma(b).
        \end{align*}
        \item[\IB] $\Att(a)\eqset \Sup(a)$ implies, by definition, $\Att(a)\geqset \Sup(a)$ and $\Sup(a)\geqset \Att(a)$. By  \ISM\ and \IAM, we obtain $\sigma(a)\leq \basescore(a)\leq \sigma(a)$. Thus, $\sigma(a) = \basescore(a)$.
        \item[\RB] $\Att(a)\eqset \Att(b)$ implies $\Att(a)\geqset \Att(b)$ and  $\Att(b)\geqset \Att(a)$. Analogously, we obtain $\Sup(a)\geqset \Sup(b)$ and  $\Sup(b)\geqset \Sup(a)$ from $\Sup(a)\eqset \Sup(b)$.
        By \RM, we thus get $\sigma(a)\leq \sigma(b)\leq \sigma(a)$.\qedhere
        
        \end{description}
\end{proof}

\propQBAFtoBSAF*
\begin{proof}
    Follows since $\setag$ corresponds to the identity function $id$ in this setting. 
\end{proof}

\section{Convergence of Elementary Modular Semantics}\label{appendix:convergence}

First, we provide the proof for the convergence in acyclic BSAFs. 
We define cycles with respect to the so-called \emph{primal graph}~\cite{DvorakKUW24} 
where each hyperedge $(\aset,a)$ is interpreted as simple edge $(e,a)$ for each $e\in \aset$.
A BSAF is acyclic if its corresponding primal graph is. 
\begin{definition}
    Let $\BF=(\aargs,\aatts,\asupps,\basescore)$.
    We define the primal graph of $\BF_P=(\aargs,\aatts',\asupps')$ with $\aatts'=\{(e,a)\mid (\aset,a)\in \aatts,e\in \aset\}$; analogously for $\asupps'$.
    $\BF$ is acyclic if $\BF_P$ is.
\end{definition}
\propConvergeneacyclic*
\begin{proof}
    We traverse the BSAF based on the topological order induced by the primal graph, starting by the leaf nodes. 
    Now, given an assumption $a$ with incoming attackers $\Att(a)$ and supporters $\Sup(a)$ so that all $b\in \aset$, $\aset\in \Att(a)\cup \Sup(a)$ have been already computed, we apply the semantics to obtain the strength of $a$. Since the BSAF is acyclic, there is no set $\aset$ with $a\in \aset$ so that $\aset$ supports or attacks any predecessor of $a$ in the primal graph.   
    Thus, the strength of $a$ is final.
\end{proof}

To prove the convergence result for the general case,
 it will be convenient to consider the (set) aggregation and influence function not for individual objects (assumptions) but for the ABAF/BSAF as a whole.
In what follows, we will thus slightly adopt the definition of the semantics.

Given a BSAF $\BF=(\aargs,\aatts,\asupps)$ corresponding to the ABAF $D=(\lit,\asm,\rules,\contraryempty)$ be an ABAF with $|\mathcal{A}|=|\aargs|=n$. 
Let $\mathcal{A}=\{a_1,\dots, a_n\}$ and let $\vec{a}$ denote the column vector ($1\times n$ matrix) whose entries correspond to the weights of the assumptions.
Let $m=\vert \{ H \mid (H,t)\in \aatts\cup\asupps\} \vert$ denote the total number of attacking and supporting sets in $D$ (resp.\ in the corresponding BSAF). 

We define the strength evolution process on the BSAF arguments (assumptions). 
The strength evolution process is a function $s:\mathbb{R}^n\rightarrow\mathbb{R}^n$ is a composition $s=\iota \circ \alpha\circ \setag$ of three functions $\setag:\mathcal{R}^n\rightarrow \mathbb{R}^m$, $\alpha:\mathbb{R}^m\rightarrow \mathbb{R}^n$, $\iota:\mathbb{R}^n\rightarrow \mathbb{R}^n$.

Let $B=(b_{ij})_{1\leq i\leq m,1\leq j\leq n}$ denote the $m\times n$ matrix with entries in $\{0,1\}$ so that the rows correspond to the sets that attack or support at least one assumption in $D$ and the entries in the row encode containment in the set. Let $i$ denote a set $H$ corresponding to the $i$-th row in $B$; 
then $b_{ij}=1_H$, i.e., $b_{ij}=1$ if $a_j\in H$ and $b_{ij}=0$ otherwise.
A set aggregation function $\setag: \mathbb{R}^n\rightarrow \mathbb{R}^m$ for w.r.t.\ $D$ computes the strength of the attacks and supports.
Table~\ref{tab:setag functions formally} gives an overview over the functions that we consider in this paper. 
\begin{table}
    \centering
    \begin{tabular}{l r}
          $\setag_\Sigma({\mathcal{A}})= B \cdot \vec{a} = ( \sum_{i=1}^n a_ib_{ij})_{1\leq j\leq m}$  & \textbf{(Sum)}\\
          $\setag_\Pi({\mathcal{A}})= (   \prod_{i=1, b_{ij}=1}^n a_i)_{1\leq j\leq m} $  & \textbf{(Prod)}\\
          $\setag_\bot({\mathcal{A}}) =(   \min \{ a_ib_{ij}\mid i\leq n \}\cup \{1\} )_{1\leq j\leq m} $ &  \textbf{(Min)}\\
          $\setag_\top({\mathcal{A}}) =(   \max \{ a_ib_{ij}\mid i\leq n \}\cup \{1\} )_{1\leq j\leq m} $ &  \textbf{(Max)}
    \end{tabular}
    \caption{Global definition of set-aggregation functions for an ABAF $D$.}
    \label{tab:setag functions formally}
\end{table}
We then apply an aggregation function $\alpha:\mathbb{R}^m\rightarrow \mathbb{R}^n$ that aggregates the attacks and supports and returns the aggregated value for each assumption, by slightly adopting the case for QBAFs (the difference is for QBAFs, $m=n$).
Let $C$ denote the $m\times n$ matrix with entries in $\{0,1,-1\}$ where each row corresponds to an attacking or supporting set and the columns to the assumptions in $D$. $C$ encodes the attackers and supporters of the assumptions. Given $a_j\in\mathcal{A}$ and a set $H\subseteq \mathcal{A}$ corresponding to the $j$-th entry of $C$, then $c_{ij}=1$ if $H$ supports $a_j$, $c_{ij}=-1$ if $H$ attacks $a_j$, and then $c_{ij}=0$ otherwise. 

\begin{definition}
A function $f$ is Lipschitz-continuous iff there exists a Lipschitz-constant $\lambda$ such that $\Vert f(x)-f(y)\Vert \leq \lambda\Vert x-y\Vert$ for all $x,y\in Dom(f)$.     
\end{definition}

We prove that our set aggregation functions $\setag$ are Lipschitz-continuous with Lipschitz-constants $\lambda_{\setag}$.
We consider the maximum norm for matrices $\Vert A \Vert =\max \{\sum_{j=1}^m |a_{ij}|\mid 1\leq i\leq n\}$ for an $m\times n$ matrix $A=(a_{ij})_{1\leq i\leq m,1\leq j\leq n}$.
Norms in $\mathbb{R}^n$ are equivalent in the sense that their difference can be bounded up to a constant factor. 

\begin{restatable}{proposition}{converdlkjgfng}
    Let $F=(\aargs,\aatts,\asupps)$ be a BSAF
    $|\aargs|=n$, $|\aatts\cup\asupps|=m$ and $h = \max \{H\subseteq \aargs\mid (H,t)\in \aatts\cup\asupps\}$.
    All considered set aggregation functions are Lipschitz-continuous, with Lipschitz-constants
    $\lambda_{\setag_\Pi}=\lambda_{\setag_\Sigma} = h$
    and $\lambda_{\setag_\top}=\lambda_{\setag_\bot}=1$.
\end{restatable}

\begin{proof}
To check that a set aggregation function $\setag: \mathbb{R}^n\rightarrow \mathbb{R}^m$ is Lipschitz-continuous it suffices to show that the component functions $\setag_j: \mathbb{R}^n\rightarrow \mathbb{R}$ are continuous. 
If all component functions $\setag_j$ are Lipschitz-continuous with Lipschitz-constant $\lambda_j$, then 
\begin{align*}
    \vert \setag_j(x) -\setag_j(y)\vert \leq \lambda_j \Vert x-y \Vert
\end{align*}
for all $j\leq m$ implies
\begin{align*}
    \Vert \setag(x) -\setag(y)\Vert & =  \max \{| \setag_j(x) -\setag_j(y) |\mid 1\leq j\leq m\} \\
    & \leq  \max \{\lambda_j \mid 1\leq j \leq m\}  \Vert x-y \Vert.
\end{align*}
\begin{itemize}
    \item[$\setag_\Sigma$:] The Lipschitz-constant is the size of the largest set $H$ that attacks or supports an assumption in $D$, i.e., 
    $\lambda_{\setag_\Sigma}=\max \{|H|\mid (H,a)\in\aatts\cup \asupps\}$. 
\begin{align*}
    \vert \setag_{\sum,j}(x) -\setag_{\sum,j}(y)\vert = &\ \vert \sum_{i=1}^n x_ib_{ij} -  \sum_{i=1}^n y_ib_{ij} \vert\\
    = &\  \vert\sum_{i=1}^n b_{ij} ( x_i-y_i  )\vert \\
    \leq &\ \sum_{i=1}^n b_{ij} \Vert x-y\Vert
\end{align*}
Thus, $\lambda_{\setag_\Sigma} = \max\{\sum_{i=1}^n b_{ij}\mid 1\leq i \leq n\}$ which corresponds to the size of the largest attacking or supporting set in $D$. 
    \item[$\setag_\Pi$:]  Likewise, $\lambda_{\setag_\Pi}=\max\{|H|\mid (H,a)\in\aatts\cup \asupps\}$:
    \begin{align*}
    \vert \setag_{\Pi,j}(x) -\setag_{\Pi,j}(y)\vert = &\ \vert\prod_{i=1, b_{ij}=1}^n x_i - \prod_{i=1, b_{ij}=1}^n y_i\vert\\
    \leq &\  \vert \prod_{i=1, b_{ij}=1}^n \max x - \prod_{i=1, b_{ij}=1}^n \max y   \vert \\
    = &\ \sum_{i=1}^n b_{ij} \Vert x-y\Vert
    \end{align*}
    \item[$\setag_\top$:] 
    The Lipschitz-constant is $1$. 
    Let $\max Z := \max \{ z_ib_{ij}\mid i\leq n \}\cup \{1\} $, $Z\in\{X,Y\}$.
    In the case $b_{ij}=0$ for all $j\leq m$, we have 
        \begin{align*}
    \vert \setag_{\top,j}(x) -\setag_{\top,j}(y)\vert = &\ \vert\max X -\max Y \vert\ = 0,
    \end{align*}
    Otherwise, we have 
    \begin{align*}
    \vert \setag_{\top,j}(x) -\setag_{\top,j}(y)\vert = &\ \vert\max X -\max Y \vert\ \\
    \leq &\ \max_{i\leq n, b_{ij}=1} \vert x_i -y_i \vert\\
    = &\ \Vert x -y \Vert
    \end{align*}
    \item[$\setag_{\bot}$:] Again, we have Lipschitz-constant $1$.
    Let $\min Z := \min \{ z_ib_{ij}\mid i\leq n \}\cup \{1\} $, $Z\in\{X,Y\}$.     
    \begin{align*}
    \vert \setag_{\bot,j}(x) -\setag_{\bot,j}(y)\vert = &\ \vert\min X -\min Y \vert\\
   \leq &\ \max_{i\leq n, b_{ij}=1} \vert x_i -y_i \vert\\
    = &\ \Vert x -y \Vert
    \end{align*}
\end{itemize}
\end{proof}

If the Lipschitz-constant $\lambda$ is strictly smaller than $1$ then the function converges with exponential speed. 
We recall the Lipschitz-constants for the aggregation function and influence functions under considerations, as given in 
Potyka~\cite[Table 1]{PotykaB24}.
We let $\lambda_\alpha=\max_{1\leq i\leq n} \lambda_\alpha^i$ where $\lambda_\alpha^i$ denotes the $i$th component of the function $\alpha$.

We say that a function $f(x)$ converges with exponential speed to a fixed point $x^*$ if $\vert f^n(x)-x^*\vert \leq cd^n$ for some $c\geq 0$ and $0\leq d<1$. 

\convergenceP*
\begin{proof}
Lipschitz-continuous functions are closed under composition with the Lipschitz-constant being the product of the composed functions. 
Sicne the strength evolution process $s$ is defined as $\iota\circ\alpha\circ \setag$, it holds that $s_D$ has Lipschitz-constant $\lambda_\setag\cdot \lambda_\alpha\cdot \lambda_\iota$.
If $\lambda_D< 1$ then $s_D$ is a contraction and converges to a unique fixed point $x^*$.

Moreover, iterating $\vert s_D(x)-x^*\vert \leq \lambda_D\vert s(x)-x^*\vert$ gives $$\vert s_D^n(x)-x^*\vert \leq \lambda_D^n \vert s_D(x) - x^*\vert.$$
Thus, $s_D$ converges with exponential speed. 
\end{proof}

\convergenceOurSems*
\begin{proof}
We compute
$\lambda_\setag\cdot\lambda_\alpha\cdot\lambda_\iota$.
By~\cite{DBLP:conf/atal/Potyka19}, we have $\lambda_{\alpha_\Pi}=\lambda_{\alpha_\Sigma}=d$ and $\lambda_{\iota_q(k)}=\lambda_{\iota_{lin}(k)}=\frac{1}{k} \max\{w,1-w\}$ for weight $w$.
Note that $\max\{w,1-w\} \leq 1$ by assumption.
Thus, for $\lambda_{\setag_\Pi}=\lambda_{\setag_\Sigma}=h$, we have 
 $h\cdot d \cdot \frac{1}{k} \max\{w,1-w\} \leq \frac{hd}{k} < 1$ iff
 $hd< k $. 
For $\lambda_{\setag_\top}=\lambda_{\setag_\bot}=1$, we get convergence whenever $d< k$. 
\end{proof}

\section{Omitted Proofs from Section~\ref{sec:instantiation}}

\PropDesiderataArg*
\begin{proof}
    \textbf{(SI)} is by definition; the remaining results follow from Proposition~\ref{prop:SatDesiderata setag}.
\end{proof}

\PropSameArg*
\begin{proof}
    By definition of flat, only sentences that are not assumptions can be derived. Hence, $\{a\}\vdash a$ is the unique argument with claim $a$ for each assumption $a\in\asm$.
\end{proof}

\PropCorrespAnrg*
\begin{proof}
    Let $F_D$ denote the BSAF corresponding to $D$. 
    Note that $F_D$ is essentially a QBAF because all attacks are of the form $(a,b)$ for assumptions $a,b\in\asm$ (induced only by the contrary function). 

    Let $\mathcal{F}_D$ denote the QBAF instanitation of $D$. 
    $\mathcal{F}_D$ has only assumption arguments, thus 
    the arguments in $\mathcal{F}_D$ and $F_D$ coincide. Moreover, attacks are also only induced by the contrary function. We obtain $\mathcal{F}_D=F_D$ and thus the semantics coincide. 
\end{proof}

\section{Details for Experiments}
\label{app:exp_details}

In this section we provide details for the empirical evaluation presented in Section~\ref{sec:experiments}. The code to reproduce all experiments and plots is available at: \url{https://github.com/briziorusso/GradualABA}

\subsubsection{Computing Infrastructure.}  
All experiments were run on an Intel(R) Xeon(R) W5-2455X CPU (up to 4.6\,GHz) with 128\,GB of RAM, using Python 3.10.12 on Ubuntu 22.04.

\subsubsection{Time Budget.}  
Each run was given a 10 minute CPU time budget. In the argument‐construction phase (Step 2 in Figures~\ref{fig:best route}\footnote{We construct attacks and supports  $(\aset,a)$ from arguments $\aset\vdash \contrary{a}$ and $\aset\vdash a$, respectively.} and \ref{fig:AF route}), 92 of the 1{,}440 ABAFs timed out (all non-flat). In the semantics phase, we observed 10 additional timeouts for BSAF versus 89 for BAF. The unsolved ABAFs tend to be the largest: among the timed-out instances the median vocabulary size is \(|\mathcal{L}|=60\) with parameters \(f=0.2\), \(r_{\max}=8\), and \(b_{\max}=8\), as detailed below.

\subsection{Data Generating Process}
\label{sec:data-generation}

We generate a suite of synthetic ABAFs by varying five design parameters:
\(
  \mathcal{L},\;
  m,\;
  r_{\max},\;
  b_{\max},\;
  f.
\)
Concretely:
\begin{description}
  \item[\(\mathcal{L}\)] Total vocabulary size (atoms \(\cup\) assumptions).  
  \item[\(m\)] Number of assumptions, set to \(0.5\,|\mathcal{L}|\). Each assumption receives a unique contrary (so \(\mathcal{A}\cap\overline{\mathcal{A}}=\emptyset\)), which preserves generality by appealing to the translation to  ABAFs with separated contraries~\cite{DBLP:journals/jair/RapbergerU23}.  
  \item[\(r_{\max}\)] Maximum rules per head. For \(r_{\max}\in\{2,4,8\}\), each sentence is derived between 1 and \(r_{\max}\) rules.  
  \item[\(b_{\max}\)] Maximum rule-body size. For \(b_{\max}\in\{2,4,8,16\}\), each rule’s body is a uniform random subset of assumptions of size in \([1, b_{\max}]\).  
  \item[\(f\)] Head-assumption fraction. We pick an additional \(\lfloor f\,m\rfloor\) assumptions to also act as heads, for \(f\in\{0.01,0.05,0.1,0.2\}\). When \(f\approx 0\) we mostly obtain flat ABAFs; larger \(f\) yields non-flat instances.  
\end{description}

All the ABAF generated are so-called \emph{atomic ABAFs}~\cite{DBLP:journals/jair/RapbergerU23}, i.e., ABAFs where each rule has only assumptions in the body. Since each ABAF can be translated into an atomic ABAF, this does not limit the generality of our experiments; it just speeds up the BSAF and BAF instantiations. 

For each configuration we generate 10 independent seeds to smooth out randomness.
Overall, we produce 1440 random ABAF, of which 960 are non-flat and 480 flat. 

\subsection{Metrics} 
\label{app:metrics_details}

We summarise the results of our experiments on the convergence using two metrics: global convergence rate and average steps to converge. Here we provide formal definitions.

Let \(\sigma_t(a)\) be the strength of assumption \(a\) at iteration \(t\). Fix thresholds \(\epsilon=10^{-3}\), \(\delta=5\), and maximum iteration cap \(T=5000\). Define for each \(a\):
\[
  \mathrm{Converged}(a)\;=\;
    \Bigl[\max_{T-\delta<t\le T} \bigl|\sigma_t(a)-\sigma_{t-1}(a)\bigr| < \epsilon\Bigr].
\]
An ABAF \emph{globally converges} iff \(\mathrm{Converged}(a)=1\) for \textit{every} assumption \(a\). Over \(N\) ABAFs, the \emph{global convergence rate} is the mean of all converged ABAFs.
Additionally, for each converged run we let
\[
  T^* 
    = \min\bigl\{\, t \le T \mid \forall  t-\delta < u \le t:
      |\sigma_u(a)-\sigma_{u-1}(a)|<\epsilon
    \bigr\},
\]
and report the \emph{average steps to converge} as the mean of all~\(T^*\).

\subsection{Additional Results}
\label{app:additional_results}

In this section we provide additional results omitted from  Section~\ref{sec:experiments} for conciseness.

\begin{figure}[t]
    \centering
    \includegraphics[width=\linewidth]{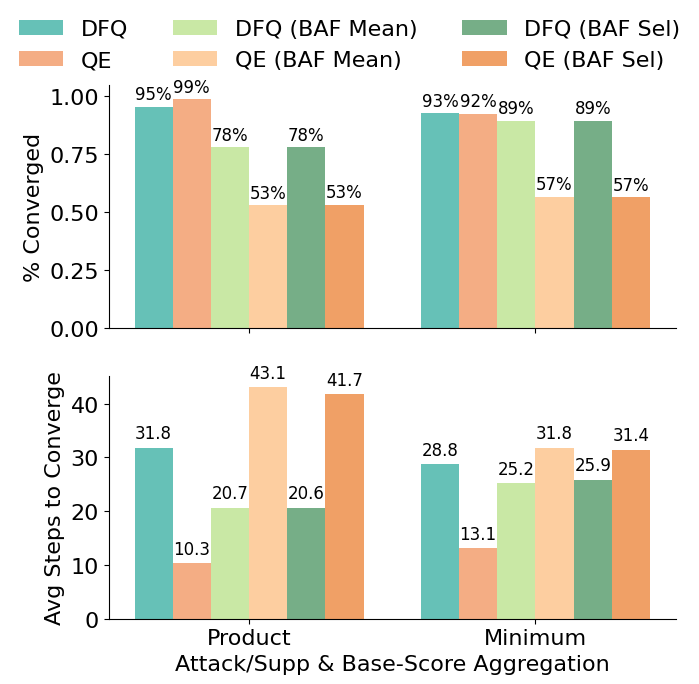}
    \caption{\textbf{BAF assumption‐strength aggregation ablation.}
    Global convergence rate (top) and average steps to converge (bottom) for DF-QuAD and QE under BAF,
    comparing mean‐aggregation \(\sigma^*_{\mathrm{mean}}\) versus selection \(\sigma^*_{\mathrm{asm}}\).
    The first two bars in each group show the corresponding BSAF results; the last four bars
    show that switching \(\sigma^*\) in BAF has negligible impact on either convergence rate or speed.}
    \label{fig:BAF-sel}
\end{figure}
\subsubsection{Assumption Strength Computation for BAF.}
To assess the impact of the assumption strength Computation function \(\sigma^*\) in BAF, we compare the mean‐aggregation
\(\sigma^*_{\mathrm{mean}}\) (averaging all attacker and supporter strengths) against a ``selection” variant
\(\sigma^*_{\mathrm{asm}}\) (selecting the trivial argument $(\{a\}\vdash a)$ for each assumption). 

Figure~\ref{fig:BAF-sel} compares convergence for DF-QuAD and QE under BAF, contrasting mean-aggregation \(\sigma^*_{\mathrm{mean}}\) vs.\ assumption selection \(\sigma^*_{\mathrm{asm}}\). The first two bars in each quadrant reproduce the BSAF results for reference; the last four bars show the four BAF variants (DF-QuAD mean, QE mean, DF-QuAD sel, QE sel). 
Results show minimal differences, within BAF, with both DF-QuAD and QE being robust to the choice of assumption-strength aggregation. 

Particularly:
\begin{itemize}
    \item Convergence Rates (top panels): Switching the assumption-strength aggregator has virtually no effect on final convergence rates for either semantics.
        \begin{itemize}
          \item \emph{Product base-score aggregation $\beta_{\Pi}$:} DF-QuAD converges in 78\% of runs under both \(\sigma^*_{\mathrm{mean}}\) and \(\sigma^*_{\mathrm{asm}}\); QE converges in 53\% under both.
          \item \emph{Minimum base-score aggregation $\beta_{\bot}$:} DF-QuAD converges in 89\% of runs under both options; QE in 57\%.
        \end{itemize}
    \item Convergence Speed (bottom panels): Minimal differences in average steps to convergence are observed.
        \begin{itemize}
          \item \emph{Product aggregation:} DF-QuAD requires 20.7 iterations with \(\sigma^*_{\mathrm{mean}}\) vs.\ 20.6 with \(\sigma^*_{\mathrm{asm}}\); QE requires 43.1 vs.\ 41.7.
          \item \emph{Minimum aggregation:} DF-QuAD needs 25.2 vs.\ 25.9 iterations; QE needs 31.8 vs.\ 31.4.
        \end{itemize}
\end{itemize}

\begin{figure}[t]
    \centering
    \includegraphics[width=\linewidth]{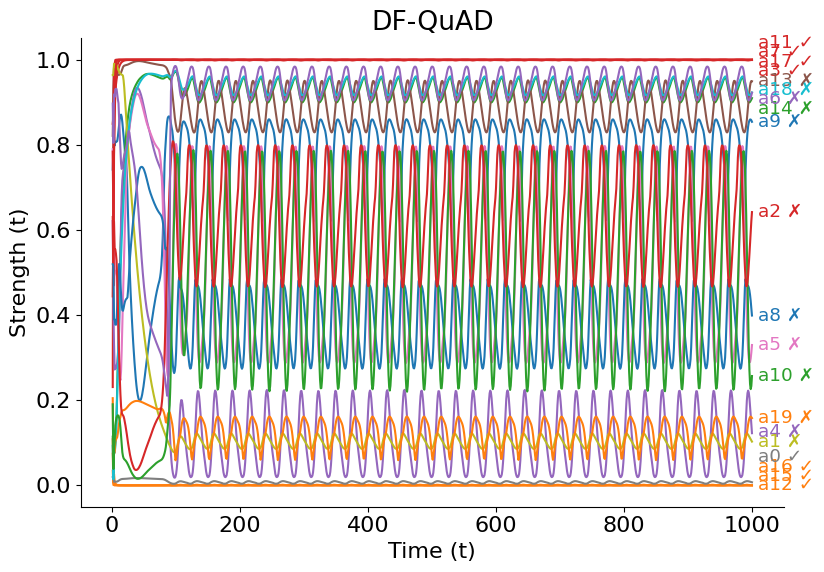}
    \caption{\textbf{DFQ trajectories at \(T=1000\).} The same experiment as in the bottom row of Figure~\ref{fig:strength_trajectories} in the main text. Assumption‐strength curves \(\sigma_t(a)\) for DF-QuAD
    on a 30‐assumption ABAF with constant base‐scores \(\tau=0.5\) and product aggregation.  Large oscillations remain after 1{,}000 iterations.}
    \label{fig:dfq_t1000}
\end{figure}

\begin{figure}[t]
    \centering
    \includegraphics[width=\linewidth]{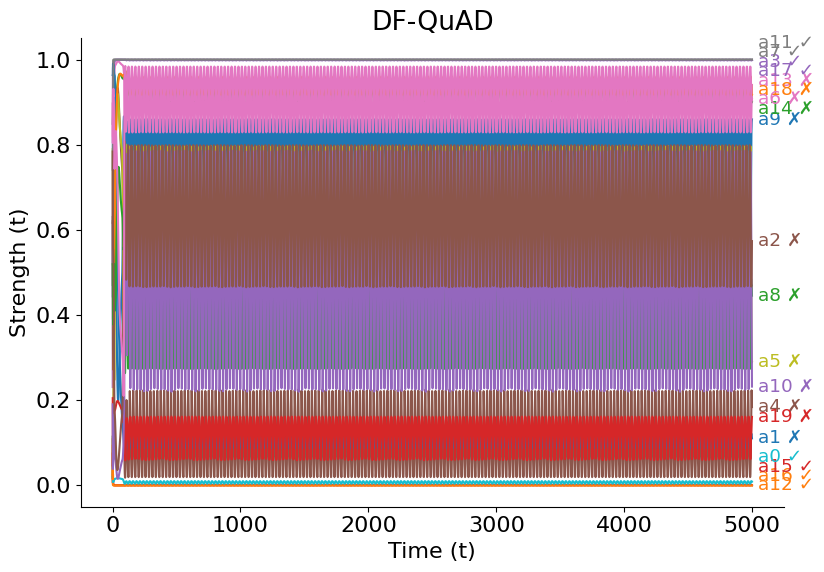}
    \caption{\textbf{DFQ trajectories at \(T=5000\).} The same experiment as Figure~\ref{fig:dfq_t1000}, now through 5{,}000 iterations:
    oscillations persist indefinitely under DF-QuAD with product aggregation.}
    \label{fig:dfq_t5000}
\end{figure}
\subsubsection{Extended Strength Trajectories.}
Here we present an extended view the example trajectory shown in Figure~\ref{fig:strength_trajectories} (bottom left quarter) in the main text. Figures~\ref{fig:dfq_t1000} and \ref{fig:dfq_t5000} display the same DF-QuAD trajectory example, run on a 30‐assumption ABAF
(\(\tau(a)=0.5\), \(\setag_{\Pi}\)) but here with \(T=1000\) and \(T=5000\) iterations, respectively, instead of only 50.  
Even after 1{,}000 iterations (Figure~\ref{fig:dfq_t1000}), the assumption‐strength curves exhibit large, sustained oscillations; these persist through
5{,}000 iterations (Figure~\ref{fig:dfq_t5000}).

\subsubsection{Sensitivity of BSAF to ABAF Structural Features.}
Figure~\ref{fig:BSAF_conv_struct} (global convergence rates) and Figure~\ref{fig:BSAF_speed_struct} (iterations to converge) compare DF-QuAD and QE under BSAF across three structural dimensions—number of assumptions, number of rules, and average rule-body size—and two base-score initialisations (\(\tau=0.5\) and \(\tau\sim U(0,1)\)). 

DF-QuAD converges in at least 90\% of runs (typically 97–100\%) in every bucket and requires 25–35 iterations. QE’s convergence varies from 53–78\% under Set–Minimum with \(\tau=0.5\) (rising to 88–95\% with random \(\tau\)) to 98–100\% under Set–Product with \(\tau=0.5\) (remaining >92\% at random \(\tau\)), and only slightly exceeds DF-QuAD in the smallest buckets ($\leq$15 assumptions). 

In speed, QE needs 10–15 iterations on simple ABAFs but slows to 34–42 on the most complex, whereas DF-QuAD stays at 25–35. Random \(\tau\) narrows QE’s worst-case by 3–5 iterations and adds 2–3 to DF-QuAD. Thus, DF-QuAD offers consistent high reliability, while QE trades rapid convergence on easy instances for reduced resilience as complexity grows.  

\subsubsection{Sensitivity of BAF to ABAF Structural Features.}
Figure~\ref{fig:BAF_conv_struct} reports global convergence rates for DF-QuAD and QE under BAF (mean base-score aggregation), bucketed by number of assumptions, rules, and average body size, with Set–Minimum and Set–Product aggregation and both \(\tau=0.5\) and \(\tau\sim U(0,1)\).  

Unlike under BSAF, both DF-QuAD and QE under BAF exhibit pronounced sensitivity to framework structure and initialisation. DF-QuAD’s convergence rate falls from near 100\% on small or medium ABAFs to roughly 90–92\% on the largest and densest ones under both Set–Minimum and Set–Product, recovering only to 94–98\% with random \(\tau\). QE performs worse still: its rate drops below 60\% under Set–Minimum and stays under 80\% even with Set–Product or random initialisation. 

Figure~\ref{fig:BAF_speed_struct} confirms larger variance in iterations to converge—DF-QuAD requires 20–30 steps (vs.\ 25–35 under BSAF), while QE ranges from $\sim$17 to $\sim$60 iterations (vs.\ 10–42 under BSAF). Random base‐score initialisation further destabilises BAF, sometimes improving QE’s poorest runs but increasing overall variance, likely due to the higher rate of timeouts. Overall, BAF yields substantially less robust and predictable convergence than BSAF for both DF-QuAD and QE.  

Overall, BSAF (Figures~\ref{fig:BSAF_conv_struct}–\ref{fig:BSAF_speed_struct}) delivers substantially higher and more consistent convergence—95–100\% and 25–35 iterations—compared to BAF’s DF-QuAD (90–98\%, 20–30 iter) and QE (50–85\%, 17–60 iter) across the same structural spectrum.
\begin{figure}[ht]
    \centering
    \includegraphics[width=\linewidth]{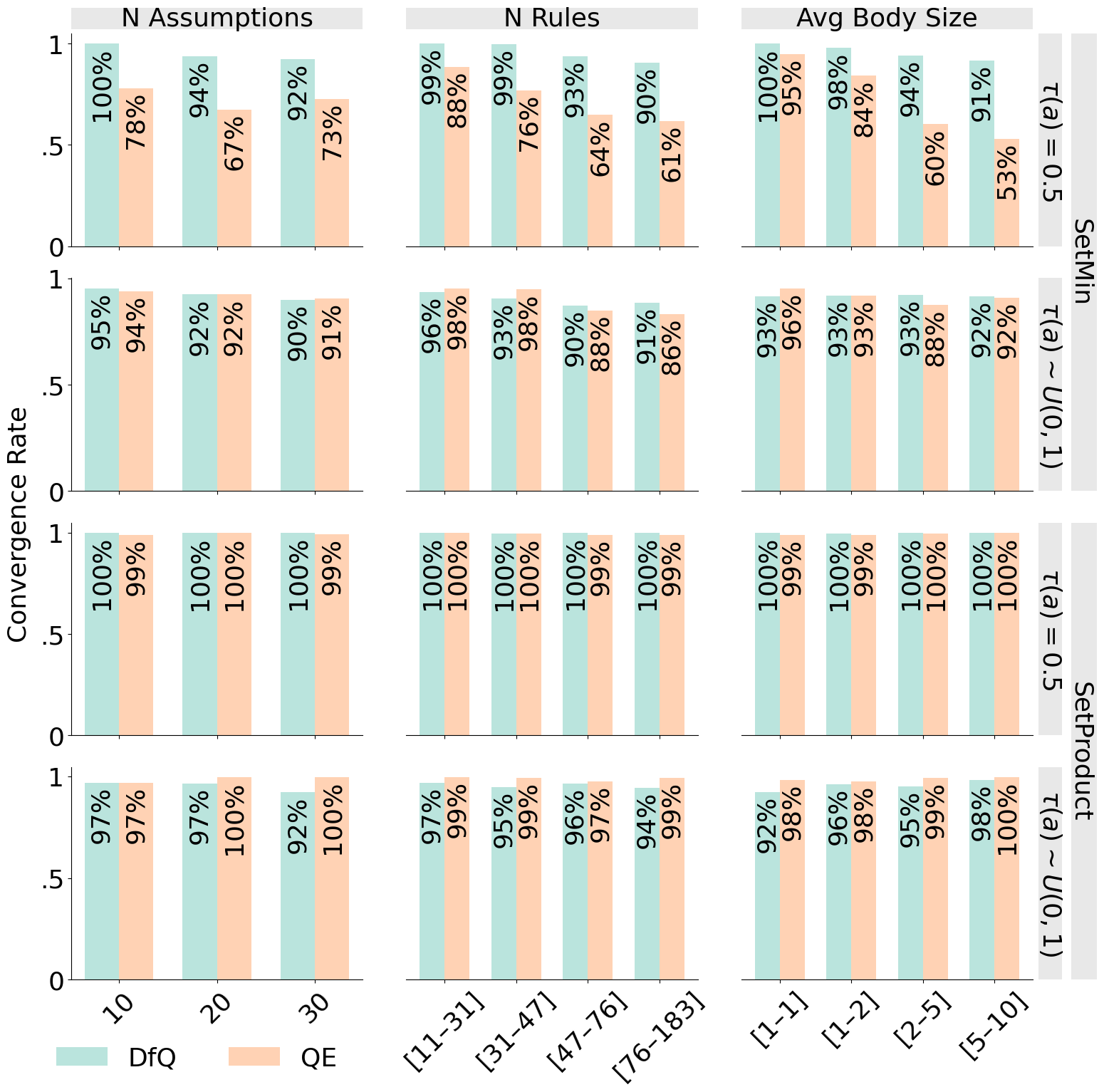}
    \caption{\textbf{BSAF convergence rate by structure.}
    Global convergence rates for DF-QuAD and QE under BSAF, bucketed by number of assumptions (left),
    number of rules (center), and average rule‐body size (right). Top row: \(\tau=0.5\); bottom row: \(\tau\sim U(0,1)\).}
    \label{fig:BSAF_conv_struct}
\end{figure}

\begin{figure}[ht]
    \centering
    \includegraphics[width=\linewidth]{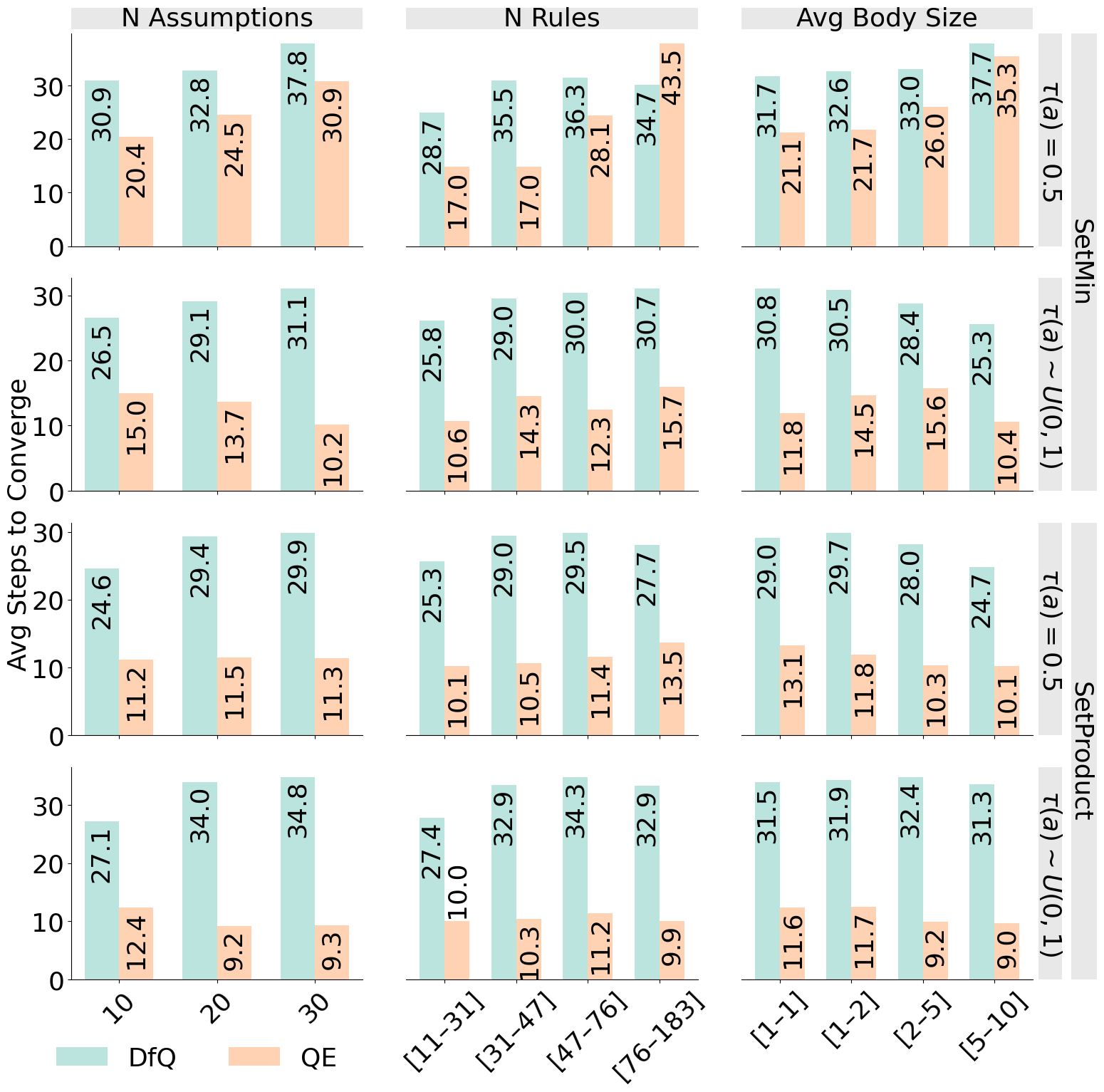}
    \caption{\textbf{BSAF convergence speed by structure.}
    Average steps to converge for DF-QuAD and QE under BSAF, across the same structural buckets and initialisations
    as Figure~\ref{fig:BSAF_conv_struct}.}
    \label{fig:BSAF_speed_struct}
\end{figure}

\begin{figure}[ht]
    \centering
    \includegraphics[width=\linewidth]{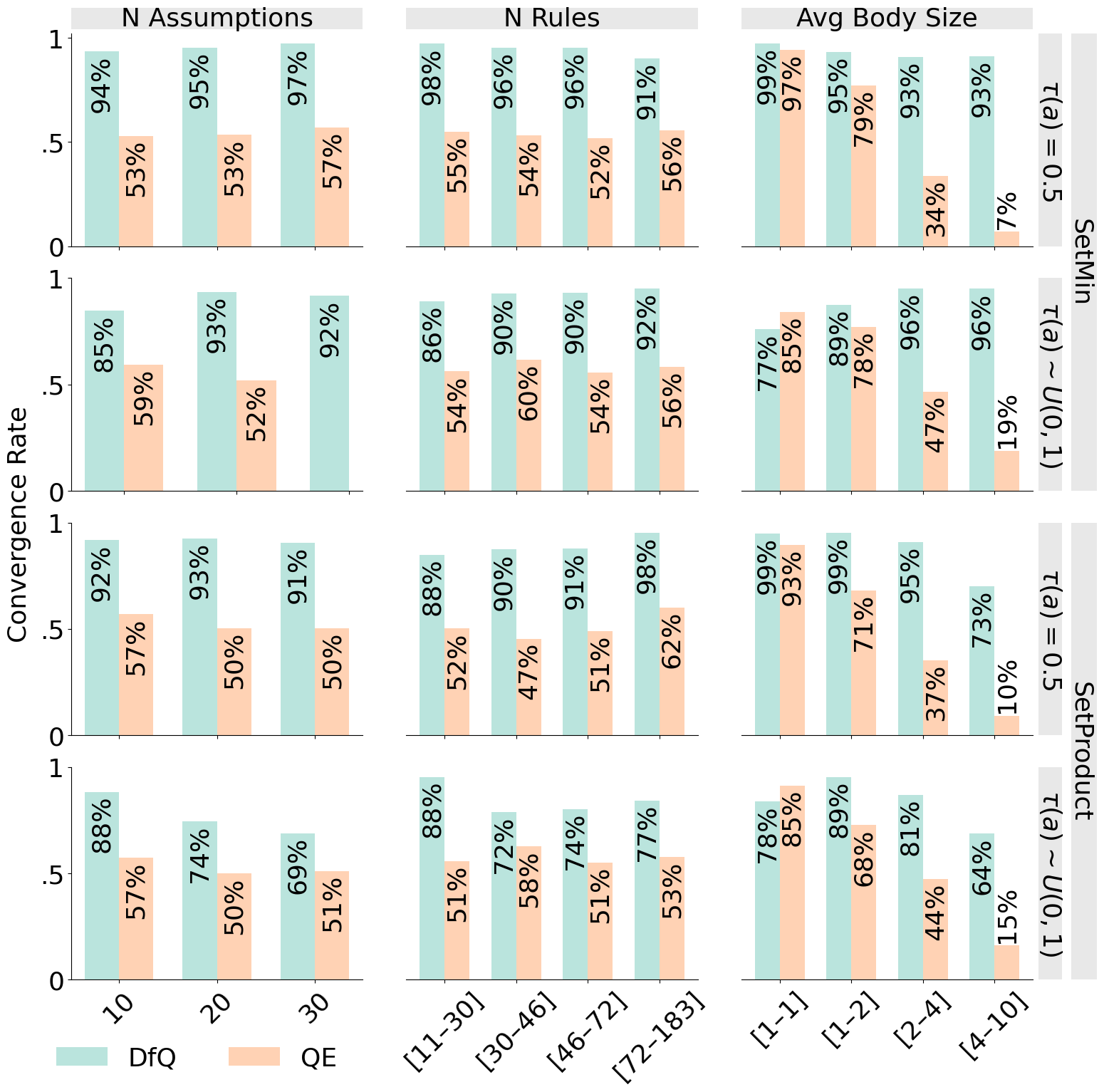}
    \caption{\textbf{BAF convergence rate by structure.}
    Global convergence rates for DF-QuAD and QE under BAF (\(\sigma^*_{\mathrm{mean}}\)), bucketed by structural features
    as before.}
    \label{fig:BAF_conv_struct}
\end{figure}

\begin{figure}[ht]
    \centering
    \includegraphics[width=\linewidth]{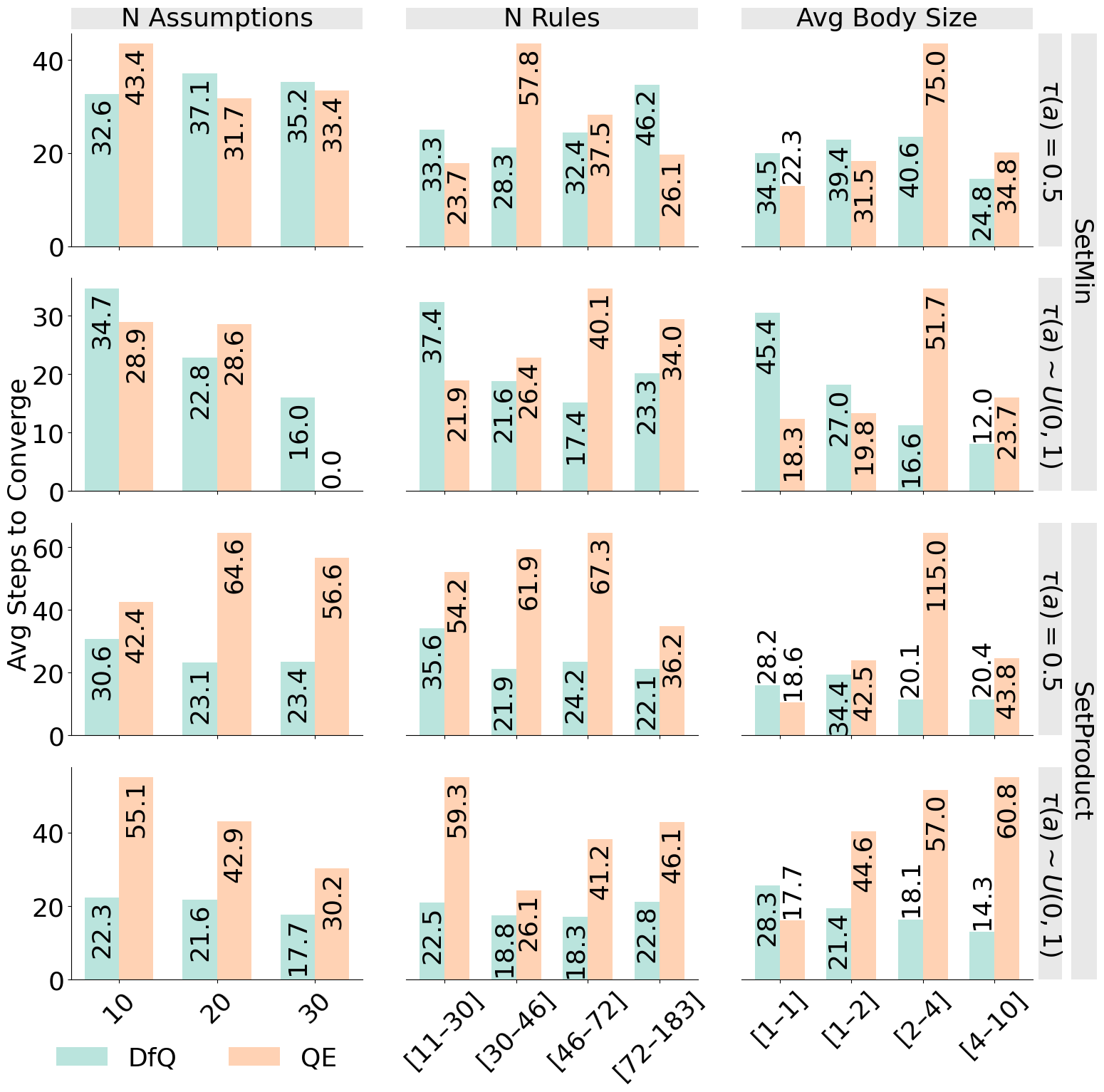}
    \caption{\textbf{BAF convergence speed by structure.}
    Average steps to converge for DF-QuAD and QE under BAF (\(\sigma^*_{\mathrm{mean}}\)), across the same buckets.}
    \label{fig:BAF_speed_struct}
\end{figure}

\end{document}